\documentclass[twoside,11pt]{article}

\usepackage{jmlr2e}
\usepackage{adjustbox}

\usepackage{lastpage}
\jmlrheading{25}{2024}{1-\pageref{LastPage}}{8/23; Revised
3/24}{5/24}{23-1075}{Boyu Li, Robert Simon Fong, and Peter Ti\v{n}o}
\ShortHeadings{Simple Cycle Reservoirs are Universal}{Li, Fong, and Ti\v{n}o}
\firstpageno{1}
\hypersetup{ hidelinks }
\usepackage{amsmath,amsxtra,hyperref}
\usepackage{ulem}
\usepackage{tikz}
\usetikzlibrary{calc,decorations.markings}
\usetikzlibrary{decorations.pathreplacing,calligraphy}
\usetikzlibrary{arrows,positioning, shapes.symbols,shapes.callouts,patterns}

\usepackage{standalone}

\usepackage{multicol}
\usepackage{enumitem}

\newcommand{\Bmath}[1]{\mbox{\bf {#1}}}

\newcommand{\Bc}{\Bmath{c}}

\newcommand{\Bx}{\Bmath{x}}

\newcommand{\By}{\Bmath{y}}
\newcommand{\Bw}{\Bmath{w}}

\newcommand{\Bv}{\Bmath{v}}

\newcommand {\e    } {\!+\!                  }

\def\x{{\Bmath x}}
\def\e{{\Bmath e}}

\numberwithin{equation}{section}

\usepackage{xcolor}
\usepackage{soul}

\newcommand{\norm}[1]{\left\lVert#1\right\rVert}

\usepackage{soul}
\newcommand{\PermutationReservoir}{Simple Multi-Cycle Reservoir}

\begin{document}

\title{Simple Cycle Reservoirs are Universal}

\author{\name Boyu Li \email boyuli@nmsu.edu \\
       \addr Department of Mathematical Sciences\\
       New Mexico State University\\
       Las Cruces, New Mexico, 88003, USA
       \AND
       \name Robert Simon Fong \email r.s.fong@bham.ac.uk \\
       \addr School of Computer Science\\
       University of Birmingham\\
       Birmingham, B15 2TT, UK
       \AND
       Peter Ti\v{n}o\email p.tino@bham.ac.uk \\
       \addr School of Computer Science\\
       University of Birmingham\\
       Birmingham, B15 2TT, UK       }
       
% \author{Boyu Li, Robert Simon Fong, Peter Ti\v{n}o}
\editor{Christian Shelton}
\maketitle

%Notation:
% P - reserved for permutations
% S - unitary matrix in diagonalization
% Q - canonical proj/embedding between C^n's

\begin{abstract}%

 {Reservoir computation models form a subclass of recurrent neural networks with fixed non-trainable input and dynamic coupling weights. Only the static readout from the state space (reservoir) is trainable, thus avoiding the known problems with propagation of gradient information backwards through time. Reservoir models have been successfully applied in a variety of tasks and were shown to be universal approximators of time-invariant fading memory dynamic filters under various settings. Simple cycle reservoirs (SCR) have been suggested as severely restricted reservoir architecture, with equal weight ring connectivity of the reservoir units and input-to-reservoir weights of binary nature with the same absolute value. Such architectures are well suited for hardware implementations without performance degradation in many practical tasks.
In this contribution, we rigorously study the expressive power of SCR in the complex domain  and show that they are capable of universal approximation of any unrestricted linear reservoir system (with continuous readout) and hence any time-invariant fading memory filter over
uniformly bounded input streams.
}
%The work concerns with topology/structure of linear reservoir systems represented by Echo State Networks, which is typically defined by a triplet consisting of a dynamical coupling matrix, input map, and a trainable readout map. In practice the dynamical coupling matrix is often randomly generated, yet in recent literature it was shown that ``simpler" structures have rival performance. One notable example being Simple Cycle Reservoirs (SCR), where the coupling matrix is one-parameter cyclic permutation matrix and the input map is simply fixed a-periodic sign patterns. We then ask the question: given an arbitrary linear reservoir system, does there exist a SCR that does the same? This work answers the question affirmatively over the complex field $\mathbb{C}$, in a constructive fashion.
\end{abstract}

\section{Introduction}

 {When learning from time series data it is necessary to adequately account for temporal dependencies in the data stream. Two main approaches emerged in the machine learning literature. In the first approach, ``time is traded for space" - under the assumption of finite memory, we collect the relevant time series history in the form of extended input that is then further processed in a static manner. Various forms of neural auto-regressive models \cite{TLin96} or transformers provide examples of this approach \cite{NIPS2017_3f5ee243}. In the second approach, we impose a parametric state-space model structure in which the state vector dynamically encodes all relevant information in the time series observed so far. The output is then again read-out from the state in the form of a static readout.
Recurrent neural networks (e.g. \cite{Downey:2017:PSR:3295222.3295354}) and 
Kalman filters \cite{Kalman1960} are examples of this work stream. 
}

 {
Of particular interest to us is a class of recurrent neural networks where the state formation and update part
of the architecture is fixed and non-trainable.
Models of this kind  \cite{Jaeger2001,Maass2002,Tino2001} are known as ``reservoir computation (RC) models"
\cite{Lukoservicius2009} with 
Echo State Networks (ESN) \cite{Jaeger2001,Jaeger2002,jaeger2002a,Jaeger2004} being one of its simplest representatives.
In its basic form, ESN is a recurrent neural network with a fixed state transition part (reservoir) and a simple trainable linear readout. In addition, the connection weights in the ESN reservoir, as well as the input weights are randomly generated. 
The reservoir weights need to be scaled
to ensure the {``Echo State Property"} (ESP): the reservoir state is an { ``echo"} of the entire input history and does not depend on the initial state. 
%Scaling reservoir weights so that the largest singular value is smaller than 1 makes the reservoir dynamics contractive and guarantees the ESP. 
 {In practice, sometimes it is the spectral
radius that is the focus of scaling, although in general, spectral radius $< 1$ does not guarantee the
ESP. However,  in this study we focus on ESNs with linear reservoir dynamics and in this case \citet{Grig2021JoGM} proved that spectral radius $< 1$ is actually equivalent to the ESP (Proposition 4.2 (i)).}
}

 {
ESNs have been successfully applied in a variety of 
tasks \cite{Jaeger2004,Bush2005,Tong2007}. 
Many extensions of the classical ESN have been suggested in the literature, e.g. deep ESN \cite{GALLICCHIO201787}, intrinsic
plasticity \cite{Schrauwen2008,Steil2007}, decoupled reservoirs
\cite{Xue2007}, 
leaky-integrator reservoir units \cite{Jaeger2007},  filter neurons with delay-and-sum readout \cite{holzmann2009} etc.
}

 {
Given the simplicity of ESN, it is natural to ask what is their representational power, compared with the general class of time-invariant fading memory dynamic filters. In a series of influential papers,
Grigoryeva and Ortega rigorously studied this question and showed the ``universality" of ESN as simple yet powerful approximators of fading memory filters \cite{grigoryeva2018echo,Grigoryeva2018}. Universal approximation capability was first established in the $L^\infty$
sense for deterministic, as well as almost surely uniformly
bounded stochastic inputs \cite{Grigoryeva2018}. This was later extended in \cite{gonon2019reservoir} to $L^p$, $1\le p<\infty$ and not necessarily almost surely uniformly
bounded stochastic inputs.
Crucially, ESN universality can be obtained even if the state transition dynamics is linear, provided the readout map is polynomial \cite{Grigoryeva2018}.
}

 {
However, the above results are existential in nature and the issue of what exactly the fixed reservoir and input-to-reservoir couplings should be remains an open problem. Indeed, the specification of such couplings requires numerous trials and even luck \cite{Xue2007} and strategies to select different reservoirs for different applications have not been adequately devised. Random
connectivity and weight structure of the reservoir is unlikely to be optimal. Moreover, imposing a constraint on the spectral radius of the reservoir matrix is a weak tool to set the reservoir parameters \cite{Ozturk2007}. 

%Ti\v{n}o and Rod\'{a}n 
\cite{rodan2010minimum} posed the following question: What is the minimal number of degrees of freedom in the reservoir design  to achieve performances on par with the ones reported in the reservoir computation literature? The answer was rather surprising: 
%Besides the reservoir size, 
It is often sufficient to consider connections forming a simple cycle (ring) among the reservoir units, all connections with the same weight. As for the input-to-reservoir coupling, in the case of linear reservoirs with non-linear readout, all that matters is the sign pattern across the input weights, the absolute value of the weights can be the same (e.g. set to 1). While such extremely constrained reservoir architectures are well suited for hardware implementations \cite{bienstman2017,Appeltant2011InformationPU,NTT_cyclic_RC}, it is less obvious to understand why simple cycle reservoirs (SCR) appear to be sufficient in many applications.
Some headway in this direction has been made along the lines of memory capacity \cite{rodan2010minimum} and temporal feature spaces \cite{Tino_JMLR_2020}.
Here we ask a different question: Can it be that simple cycle reservoir structures are actually universal in the sense outlined above? We will show that in the complex domain (cyclic coupling of reservoir units with the same (complex) weight; input-to-reservoir couplings constrained to $\pm1$, $\pm i$) the answer is ``yes"! If we constrain the input-to-reservoir couplings to $\pm1$, a twin SCR is needed with two reservoir cycles operating in parallel on the same input stream.
}

\section{The Setup}
%  {
% \begin{enumerate}
%     \item Changed everything over $\mathbb{C}$ because when we diagonalize unitary matrix, we may need complex matrix.
%     \item full-cycle permutation? circular shift?
% \end{enumerate}
% }

 {Let us first briefly recall the notion of fading memory property in the context of our study. Consider input-output systems (filters) that map  $\{\Bc_t\}_{t\in\mathbb{Z}_-} \subset \mathbb{C}^m$ to $\{\By_t\}_{t\in\mathbb{Z}_-} \subset \mathbb{C}^d$ with the imposition that at each $t$ the output $\By_t$
is determined only by the past inputs $\{\Bc_{t'}\}_{t' \le t}$. Such systems are called causal.

We would now like to characterise situations where the influence of inputs from deeper past on the present output is gradually fading out. In other words, under such input-output systems, given a time instance $t$, two input sequences $\Bc$ and $\Bc'$ having ``similar recent histories" up to time $t$
(but not necessarily the deeper past ones) 
will yield ``similar outputs" at $t$. This can be formalised for example through topological arguments as follows (see e.g. \citep{Grigoryeva2018}):
Consider a norm $\norm{\cdot}$ on $\mathbb{C}^m$. The infinite product space consisting of left-infinite sequences $\{\Bc_t\}_{t\in\mathbb{Z}_-}$ can be endowed with a Banach space structure by considering e.g. the supremum norm assigning to each sequence
$\Bc = \{\Bc_t\}_{t\in\mathbb{Z}_-}$ the norm
\[
\norm{\Bc}_\infty := \sup_{t\in\mathbb{Z}_-} \norm{\Bc_t}.
\]
If we wanted to assign more weight on the recent items than on the past ones, we could modify the supremum norm into a weighted norm, 
\[
\norm{\Bc}_{\Bw} := \sup_{t\in\mathbb{Z}_-} \norm{w_t \cdot \Bc_t},
\]
where 
$\Bw = \{w_t\}_{t\in\mathbb{Z}_-}$ is a weighting sequence, i.e. a strictly decreasing sequence (in the reverse time order) of positive real numbers with a fixed maximal element (e.g. $w_0=1$). The space 
%(\mathbb{C}^m)^{\mathbb{Z}_-}$
$\{ \Bc \in (\mathbb{C}^m)^{\mathbb{Z}_-} | \ \norm{\Bc}_{\Bw} < \infty\}$
equipped with the weighted norm $\norm{\cdot}_{\Bw}$ forms a Banach space\footnote{
In the case of uniformly bounded inputs  (the setting of this study), all left-infinite sequences have finite weighted norm.}
\citep{grigoryeva2018echo}.
Analogously, given a (possibly different) weighting sequence $\Bv = \{v_t\}_{t\in\mathbb{Z}_-}$, we  
define a norm on the output left-infinite sequences $\By = \{\By_t\}_{t\in\mathbb{Z}_-}$,
$\norm{\By}_{\Bv} := \sup_{t\in\mathbb{Z}_-} \norm{v_t \cdot \By_t}$ (the norm $\norm{\cdot}$ is this time defined on $\mathbb{C}^d$). 
We now require
% This can be formalised by requiring 
the maps realised by the causal input-output systems be continuous with respect to the topologies generated by the weighted norms $\norm{\cdot}_{\Bw}$ and $\norm{\cdot}_{\Bv}$. Such input-output systems are said to have the {\it fading memory property} (FMP)\footnote{
We note that the systems we study in this paper will have a stronger FMP, in particular, the $\lambda$-exponential FMP, for some $0<\lambda<1$, where the elements of the weighting sequence are given by $w_t = e^{\lambda t}$. However, for the case of uniformly bounded inputs we study here, the $\lambda$-exponential FMP implies FMP for any weighting sequence \citep{Grigoryeva2018}.}.  
}

 {We now proceed by introducing the basic building blocks needed for the developments in this study.}

\begin{definition}\label{def.lrc} A \textbf{linear reservoir system} is formally defined as the triplet $R:= (W,V,h)$ where the \textbf{dynamic coupling} $W$ is an $n\times n$ matrix, the \textbf{input-to-state coupling} $V$ is an $n\times m$ matrix, and the state-to-output mapping (\textbf{readout}) $h:\mathbb{C}^n \to \mathbb{C}^d$ is a (trainable) continuous function. 

%\comment{Uniform continuity of $h$: why is this assumption necessary? Can we make this assumption in general? In practice the readout map is mostly linear/ridge regression; linear maps are unif. cts. In general I think no assumptions are made on the readout map but it was sort of assumed to be ``anything"? 
%Re: we assume $c_t$ is uniformly bounded, so $x_t$ should also be bounded. Then $h$ is kind of defined on a compact set and therefore uniform cts should follow from cts Re.Re. Ah I see if its not an assumption I think should write it that way then in case we get questioned}. 

%\comment{Uniform continuity of $h$: why is this assumption necessary? Can we make this assumption in general? In practice the readout map is mostly linear/ridge regression; linear maps are unif. cts. In general I think no assumptions are made on the readout map but it was sort of assumed to be ``anything"? 
%Re: we assume $c_t$ is uniformly bounded, so $x_t$ should also be bounded. Then $h$ is kind of defined on a compact set and therefore uniform cts should follow from cts Re.Re. Ah I see if its not an assumption I think should write it that way then in case we get questioned}. 

The corresponding linear dynamical system is given by:
\begin{equation} \label{eq.system}
   \begin{cases} \Bx_t &= W \Bx_{t - 1} + V \Bc_t \\
    \By_t &= h(\Bx_t)
    \end{cases}
\end{equation}
where $\{\Bc_t\}_{t\in\mathbb{Z}_-} \subset \mathbb{C}^m$, $\{\Bx_t\}_{t\in\mathbb{Z}_-} \subset \mathbb{C}^n$, and $\{\By_t\}_{t\in\mathbb{Z}_-} \subset \mathbb{C}^d$ are the external inputs, states and outputs, respectively.
We abbreviate the dimensions of $R$ by $(n,m,d)$.

We make the following assumptions for the system:
\begin{enumerate}
    \item $W$ is assumed to be strictly \textbf{contractive}. In other words, its operator norm $\norm{W}<1$. The system \eqref{eq.system} thus satisfies the fading memory property (FMP).
    \item We assume the input stream is $\{\Bc_t\}_{t\in\mathbb{Z}_-}$ is \textbf{uniformly bounded}. In other words, there exists a constant $M$ such that $\norm{\Bc_t}\leq M$ for all $t\in\mathbb{Z}_-$. 
\end{enumerate}
 {The contractiveness of $W$ and the uniform boundedness of input stream imply that the images of the inputs $\Bc \in (\mathbb{C}^m)^{\mathbb{Z}_-}$ under the linear reservoir system live in a compact space $X \subset{\mathbb{C}}^n$. With slight abuse of mathematical terminology we call $X$ a \textbf{state space}.}
\end{definition}

% Under these assumptions, for each 
% $\{c_t\}_{t\in\mathbb{Z}_-}$, the system~\ref{eq.system} has a unique solution given by 
% \begin{align*}
%     %x_0  = 
%     \Bx_t &=\sum_{n\geq 0} W^n V \Bc_{t-n}, \\
%     \By_t &= h(\Bx_t). 
% \end{align*}
% {[I STILL THINK, GIVEN EQ. \ref{eq.system}, THAT $x_0$ IS CONFUSING]}
% {Removed it! No idea what that x_0 was for.}
%  {I AGREE, WE CAN SAY E.G.:\\
Under the assumptions outlined above in Definition~\ref{def.lrc}, for each left infinite time series
$c = \{\Bc_t\}_{t\in\mathbb{Z}_-}$, the system~\eqref{eq.system} has a unique solution given 
by
\begin{align*}
    {\Bx_t}(c) &=\sum_{n\geq 0} W^n V \Bc_{t-n}, \\
    {\By_t}(c) &= h({{\Bx_t}(c)}). 
\end{align*}
To ease the mathematical notation we will refer to the solution simply as $\{(\Bx_t,\By_t)\}_t$.
% \\THIS SHOULD MAKE OUR PRESENTATION MORE READABLE. NO NEED FOR HATS! ;-)}

\begin{definition} For two reservoir systems $R=(W,V,h)$ (with dimensions $(n,m,d)$) and $R'=(W', V', h')$ (with dimensions $(n',m,d)$): 
\begin{enumerate}
    \item We say the two systems are \textbf{equivalent} if for any input stream, the two systems generate the same output stream. More precisely, for any input $c=\{\Bc_t\}_{t\in\mathbb{Z}_-}$
    , the {solutions $\{(\Bx_t,\By_t)\}_t$ and $\{(\Bx'_t,\By'_t)\}_t$ for systems $R$ and $R'$, given by:
    \begin{align*}
        {\By_t} &= h \left(\Bx_t(c)\right) = h\left(\sum_{j\geq 0} W^j V \Bc_{t-j}\right) \ \ \hbox{and}\\
        {\By_t'} &= h' \left(\Bx'_t(c)\right) = h'\left(\sum_{j\geq 0} \left(W'\right)^j V' \Bc_{t-j}\right),
    \end{align*}
    respectively,} satisfy ${\By}_t={\By}_t'$ for all $t$.
    % {NOTE THAT WE USE $n$ TO DENOTE STATE SPACE DIMENSIONALITY.}
    \item For $\epsilon>0$, we say the \textbf{two systems are $\epsilon$-close} if the outputs of the two systems, given any input stream, are $\epsilon$-close. That is (under the notation above),  $\norm{{\By}_t-{\By}_t'}_2<\epsilon$ for all $t$.
    %  {SHOULDN'T WE SPECIFY THE NORM HERE? OR AT LEAST SAY THAT GIVEN A NORM $\norm{\cdot}$ ON $\mathbb{R}^d$, WE SAY THAT THE TWO SYSTEMS ARE $\epsilon$-CLOSE... }  {Re: Used 2-norm for now.}
\end{enumerate}
\end{definition}

\begin{remark}
Since norms on finite-dimensional spaces are equivalent, we can replace the $2$-norm in the definition of $\epsilon$-close by any other norm on ${\mathbb{C}}^d$. Our main results do not depend on the particular choice of the norm. All subsequent norms over scalar fields are $2$-norms unless specified otherwise.
\end{remark}

 {
\begin{remark}
Our notion of system equivalence differs from the more structural (iso)morphism approaches sometimes taken in control and systems theory. For example, the system equivalence in \cite{Grig2021JoGM} is treated as system isomorphism, that is, equivalence not only in terms of filter map equivalence - ``the same inputs leading to the same outputs'' - but also equivalence in terms of preservation of the internal dynamics. In particular, it involves a map between the states of the two systems: two systems are isomorphic if there is a bijection $f$ that maps the states of one system $R_1$ to the states of another system $R_2$ such that $f$ preserves the state evolution, as well as the readout. 
\end{remark}
}
        
For the rest of the section, we outline the main results of the paper. We begin by following definitions:

\begin{definition} Let $P=[p_{ij}]$ be an $n\times n$ matrix.
\begin{enumerate}
    \item We say $P$ is a \textbf{permutation matrix} if there exists a permutation $\sigma$ in the symmetric group $S_n$ such that $p_{ij}=\begin{cases} 1, &\text{ if }\sigma(i)=j, \\0, &\text{ if otherwise.}\end{cases}$
    \item We say a permutation matrix $P$ is a \textbf{full-cycle permutation}\footnote{Also called left circular shift or cyclic permutation in the literature} if its corresponding permutation $\sigma\in S_n$ is a cycle permutation of length $n$. 
\end{enumerate}
\end{definition}
Let $P$ be an $n\times n$ permutation matrix associated with $\sigma\in S_n$, and let $\{e_i\}_{i=1}^n$ be the canonical basis for $\mathbb{C}^n$. One can easily verify that $Pe_i = e_{\sigma(i)}$, which defines a permutation of the basis vectors by $\sigma$.  

We call a matrix $W$ \textbf{a contractive permutation} (resp. \textbf{a contractive full-cycle permutation} if $W=aP$ for some  scalar $a\in (0,1)\subset \mathbb{R}$ and $P$ is a permutation (resp. full-cycle permutation). % with $|c|<1$

%  {Rod\'{a}n and Ti\v{n}o 
 {\cite{rodan2010minimum} introduced a minimum complexity reservoir system with  a contractive full-cycle permutation dynamical coupling matrix. In the following definition, we recall its linear form, as well as its extension to input-to-state coupling in the complex domain:}

\begin{definition}
\label{def.rc}
    A linear reservoir system $R = \left(W,V,h\right)$ with dimensions $(n,m,d)$ is called:
    \begin{itemize}

    \item A \textbf{Simple Cycle Reservoir (SCR)}\footnote{
     {Note that in \cite{rodan2010minimum} there is an additional requirement that the sign pattern in the binary input-to-state coupling $V$ is a-periodic. Also, although all the input weights have the same absolute value, it does not have to be 1. The sign aperiodicity and scaling of the input weights  are not needed for the developments in this study.}} if: % this is the SCR in literature. if 
    \begin{enumerate}
        \item $W$ is a contractive full-cycle permutation, and
        \item $V \in \mathbb{M}_{n \times m}\left(\left\{-1,1\right\}\right)$. 
    \end{enumerate}

    \item A \textbf{Complex Simple Cycle Reservoir ($\mathbb{C}$-SCR)} 
    %\comment{RC is reserved for Reservoir Computing, maybe more explicitly $\mathbb{C}$-SCR?} 
    if:
    \begin{enumerate}
        \item $W$ is a contractive full-cycle permutation, and
        \item $V \in \mathbb{M}_{n \times m}$ and all entries of $V$ are either $\pm 1$ or $\pm i$ . 
    \end{enumerate}

    \end{itemize}  

\end{definition}

 {We also introduce a composite reservoir structure with multiple contractive full-cycle permutation couplings and binary input weights:  }

 {
\begin{definition}
\label{def.mcr}
For $k > 1$, 
    a linear reservoir system $R = \left(W,V,h\right)$ with dimensions $(n,m,d)$ is called a
 {\bf Multi-Cycle Reservoir of order $k$}
if:
    \begin{enumerate}
        \item $W$ is block-diagonal with $k$ (not necessarily identical) blocks of   contractive full-cycle permutation couplings $W_1$, ..., $W_k$, of dimensions
        $n_i \times n_i$, $i=1,2,...,k$,
 \begin{align*}
     W := \begin{bmatrix} W_1 & & &  \\ 
     & W_2 & &  \\
     & & \ddots &  \\
     & & & W_k \end{bmatrix}, 
     &\quad 
     \sum_{i=1}^k n_i = n,
 \end{align*}
        and 
        \item $V \in \mathbb{M}_{n \times m}\left(\left\{-1,1\right\}\right)$. 
    \end{enumerate}

\end{definition}
}
 {
The state $\x \in \mathbb{C}^n$ of such a multi-cycle system is composed of the $k$ component states $\x^{(i)} \in \mathbb{C}^{n_i}$,
$i=1,2,...,k$, 
$\x = (\x^{(1)}, ..., \x^{(k)})$.
In our case, the readout will act on a linear combination of the component states,
\[
h(\x) = h\left(\sum_{i=1}^k a_i \cdot \x^{(i)} \right),
\]
where $a_i \in \mathbb{C}$ are mixing coefficients.
}

 {
Of particular interest will be  simple multi-cycle structures with identical full-cycle blocks:

\begin{definition}
\label{def.smcr}
    A linear reservoir 
    %system $R = \left(W,V,h\right)$ with dimensions $(n,m,d)$ 
    is called a
{\bf Simple Multi-Cycle Reservoir (SMCR) of order $k$} if it is a Multi-Cycle Reservoir of order $k$ with $k$ \emph{identical} (contractive full-cycle permutation) blocks.
% \textbf{\PermutationReservoir (SMCR)} 
%if:
%    \begin{enumerate}
%        \item $W$ is block-diagonal with identical contractive full-cycle permutation blocks, and 
%        \item $V \in \mathbb{M}_{m \times n}\left(\left\{-1,1\right\}\right)$. 
%    \end{enumerate}

\end{definition}
}

 {
Finally, we introduce a minimal version of  the multi-cycle reservoir system with just two (not necessarily identical) blocks:

\begin{definition}
\label{def.twin}
    A linear reservoir 
    %system $R = \left(W,V,h\right)$ with dimensions $(n,m,d)$ 
    is called a
{\bf Twin Simple Cycle Reservoir (Twin SCR)} if it is a Multi-Cycle Reservoir of order $2$.
\end{definition}
}

 {We note that while this study considers contractive dynamics (by requiring the operator norm of the state space coupling to satisfy $\norm{W}<1$), for the FMP to hold one only needs a weaker condition on $W$ involving the spectral radius, namely $\rho(W) < 1$. 
FMP under the stronger condition $\vert\vert W \vert\vert < 1$ has been established in \cite{Jaeger2010ErratumNF} or \cite{JMLR:v20:19-150}(Theorem 7). Under the weaker condition $\rho(W) < 1$ it can be shown using the ESP property in \cite{Grig2021JoGM}(Proposition 4.2 (i)) together with \cite{manjunath2022embedding}(Theorem 3). 
}

 {
\subsection{Summary of Essential Notations}

We conclude this section by summarizing 
%the summarizes 
the essential mathematical notations used in the subsequent sections.
%discussions. 

Fields are denoted by  blackboard-bold capital letters, for example, $\mathbb{R}, \mathbb{C}$ and $\mathbb{K}$ denote real numbers, complex numbers and arbitrary field respectively. In addition, we denote by $\mathbb{T} \subset \mathbb{C}$ the unit circle in $\mathbb{C}$. The set $\mathbb{M}_{m \times n}(\mathbb{K})$ contains all $m$-by-$n$ matrices over the field $\mathbb{K}$. Given an $n$-by-$m$ matrix $W \in \mathbb{M}_{n\times m}$, $\norm{W}$ denotes its operator norm. Hilbert spaces are written in calligraphic font such as $\mathcal{H}$.

Capital letters typically denote matrices with some symbols reserved for specially structured matrices. Examples include $P$ for permutation matrices, $U$ for unitary matrices, and D for diagonal matrices (unless specified otherwise). The exceptions are:
\begin{itemize}
    \item $M$ -- the uniform upper bound of input stream.
    \item $J:\mathbb{C}^{n} \hookrightarrow \mathbb{C}^{n'}$ -- the canonical embedding 
of $\mathbb{C}^n$ onto the first $n$-coordinates of $\mathbb{C}^{n'}$ with $n' \geq n$, and
    \item $Q_n:\mathbb{C}^{n_1} \hookrightarrow \mathbb{C}^n$ -- the canonical projection of the first $n$ coordinates ($n_1 \geq n$).
\end{itemize} 

Bold lower case letters such as $\Bv \in \mathbb{K}^n$ denote column vectors with the assumption that $\mathbb{K} = \mathbb{C}$ throughout the paper, unless specified otherwise. For vectors $\Bv \in \mathbb{K}^n$, $\norm{\Bv} = \norm{\Bv}_2$ denotes its Euclidean norm. Vector valued left infinite time series are indexed with a subscript $t$, denoted by $\{\Bx_t\}_{t\in\mathbb{Z}_-} \subset \mathbb{C}^n$. For vectors $\Bx \in \mathbb{C}^n$, the entries are denoted by $\left(x_1,x_2,\ldots,x_n\right)^\top$. In the case where there are multiple reservoir systems involved, the state vectors of the $i^{th}$ reservoir system would be indicated with an additional upper script $\Bx_t^{(i)} \in \mathbb{C}^n$.
}

\section{Summary of Main Results}

 {Motivated by the minimum complexity reservoir architecture \cite{rodan2010minimum}, our main goal is to study the universality properties of such radically constrained reservoir structures with simple cyclic interconnections in the dynamic coupling and binary input weights. 

  {The flow of our argumentation is summarized in Figure~\ref{fig:fullpaper_simple} below. Each arrow in the diagram denotes an approximation step. The symbol $\prec$ indicates an increase in the approximant state space dimensionality.}
 In particular, 
 we show (Theorem~\ref{thm.main.cscr}) that {\em any} linear reservoir system \eqref{eq.system} can be approximated by a Complex Simple Cycle Reservoir. The situation changes if we wanted to constrain the input weights solely to $\{-1,+1\}$ as in the minimum complexity reservoirs \cite{rodan2010minimum}, while maintaining the universal approximation capabilities.
One can either use a Twin Simple Cycle Reservoir (Theorem~\ref{thm.main.rscr}), or further insist on identical cyclic reservoir blocks (which may be advantageous from a hardware implementation point of view), in which case a Simple Multi-Cycle Reservoir of order greater than 2 may be needed (Theorem~\ref{thm.main.pr}). 
 {The details of each approximation step will be fleshed-out and summarized in more technical manner in Section~\ref{sec:summary}.}

\begin{center}
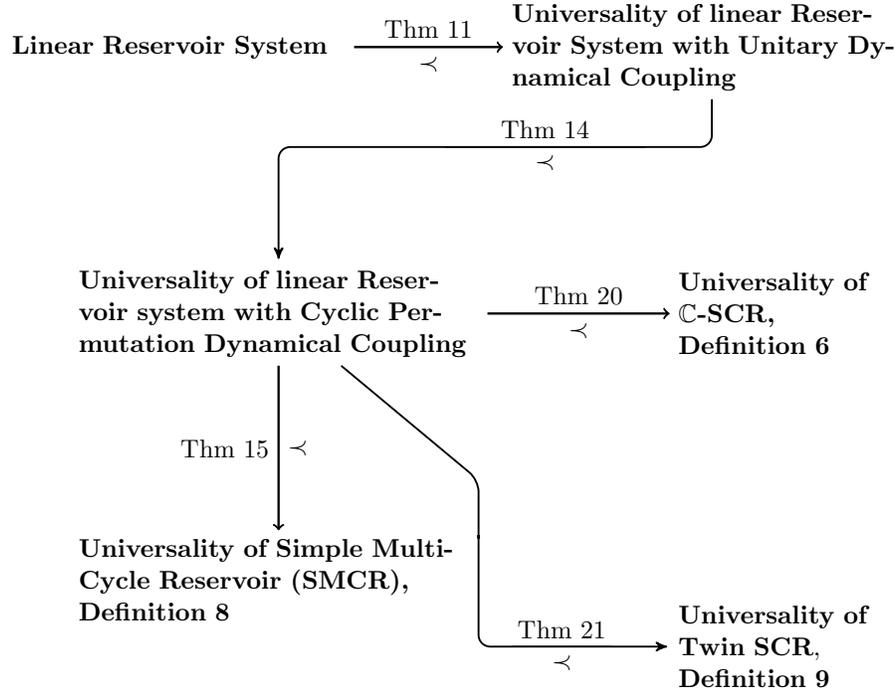
\begin{figure}[ht!]
\centering
%  \includestandalone[width=0.8\textwidth]{tikz/full_flow_simple}%     without .tex extension
  % or use \input{mytikz}
  \begin{adjustbox}{width=0.8\textwidth}
  \begin{tikzpicture}[
    pre/.style={=stealth',semithick},
    post/.style={->,shorten >=1pt,>=stealth',thick}
    ]

 % Nodes
 \node[black, text width=5cm] (orig) at (-7,3) {\textbf{Linear Reservoir System}};
 
 \node[black, text width=6cm](unit) at (1,3) {\textbf{Universality of linear Reservoir System with Unitary Dynamical Coupling}};
 
 \node[black, text width=6cm] (perm) at (-5.5,-1){\textbf{Universality of linear Reservoir system with Cyclic Permutation Dynamical Coupling}};

 \node[black, text width=6cm] (permpm1) at (-5.5,-5) {\textbf{Universality of Simple Multi-Cycle Reservoir (SMCR),} \\ \textbf{Definition~\ref{def.smcr}}};

 \node[black, text width=3cm] (cscr) at (2,-1) {\textbf{Universality of $\mathbb{C}$-SCR,} \\ \textbf{Definition~\ref{def.rc}}};

 \node[black, text width=3cm] (rscr) at (2,-11 + 5) {\textbf{Universality of Twin SCR}, \\ \textbf{Definition~\ref{def.twin}}};
 
 % Arrows
 \path [->,thick] (orig) edge node [above] {Thm \ref{thm.dilation}} (unit);
 \path [->,thick] (orig) edge node [below] {$\prec$} (unit);

 % \draw[post,rounded corners=5pt] (unit)-|(below_orig_perm);
 % \draw[post,rounded corners=5pt] (below_orig_perm)|-(perm);
 
 \draw[post,rounded corners=5pt] (unit) -- node[below] {} ++(0,-2 + 0.5) -| node[below] {} ++(-6.5 ,-0.5) -- (perm);
 \draw[-, ultra thin] (-3,1 + 0.5) --node [below] {$\prec$}  (0, 1 + 0.5);
 \draw[-, ultra thin] (-3, 1 + 0.5) --node [above] {Thm \ref{thm.to.permutation}}  (0, 1+ 0.5);
 % \draw[post,rounded corners=5pt] (orig)|-(rscr); % final arrow

 % Arrows
 \path [->,thick] (perm) edge node [left] { Thm \ref{thm.main.pr}} (permpm1);
 \path [->,thick] (perm) edge node [right] {$\prec$} (permpm1);

 \path [->,thick] (perm) edge node [above] {Thm \ref{thm.main.cscr}} (cscr);
 \path [->,thick] (perm) edge node [below] {$\prec$} (cscr);
 
 % \path [->] (cscr) edge node [left] {Thm \ref{thm.main.2}} (rscr);
 % \path [->] (cscr) edge node [right] {$\prec$} (rscr);

 \draw[post,rounded corners=5pt] (perm) -- node[below] {} ++(3,-3.5 + 1) |- node[above] {} ++(0 ,-2  + 1) |-(rscr);
 \draw[-, ultra thin] (-2,-11+ 5) --node [below] {$\prec$}  (-0.5 ,-11+ 5);
 \draw[-, ultra thin] (-2,-11 + 5) --node [above] {Thm \ref{thm.main.rscr}}  (-0.5 ,-11+ 5);
 
 % \draw[->] (2.5,3) --node [above] {Prop. \ref{prop.perturb.unitary}} (4,3);
 % \draw[->] (2.5,3) --node [above] {Prop. \ref{prop.perturb.unitary}} (4,3);

\end{tikzpicture}
\end{adjustbox}
  \caption{Flow of the main results of this paper}
  \label{fig:fullpaper_simple}
\end{figure}
\end{center}
 
Crucially, these results enable us to connect to the work of  {\cite{Grigoryeva2018}(Corollary 11)} proving that linear reservoir systems with polynomial readouts are universal: any time-invariant fading memory filter can be approximated by a linear reservoir system. 
Since we show that our extremely constrained 
reservoir cyclic structures can approximate to arbitrary precision any linear reservoir system with continuous (hence also polynomial) readout, we end up with a rather surprising conclusion: 
Any time-invariant fading memory filter can be approximated to arbitrary precision by (i) a $\mathbb{C}$-SCR, (ii) a SMCR, and (iii) a Twin SCR. Yet, besides the reservoir size, they all have a {\em single degree of freedom in the reservoir dynamic coupling} $W$ - the cycle weight parameter $\lambda$ (spectral radius of $W$). To summarize: $\mathbb{C}$-SCR, SMCR,  and Twin SCR are all universal reservoir structures with a single tunable reservoir weight parameter.

Last, but not least, all our poofs are constructive. Hence, given any linear reservoir structure, we can explicitly construct its simple approximator with a single tunable parameter, which can be of vital importance in hardware implementations of reservoir systems \cite{bienstman2017,Appeltant2011InformationPU,NTT_cyclic_RC}.
}

\section{Unitary Dilation of Linear Reservoirs} \label{sec:unitary}

% \comment{the $T$, $S$, and $U$ are swapped in the next section making it a bit confusing, I'll swap it later. Here $T$ should be $W$, and $S$ is $T$. Better if we change the subsequent section according to the dilation literature...}

 {As an intermediate step towards constructing $\mathbb{C}$-SCR approximators we first establish approximation capabilities of linear reservoir systems with unitary dynamic coupling $W$.  
In particular, using the Dilation Theory we show that given any linear reservoir system and $\epsilon > 0$, we can construct an $\epsilon$-close linear reservoir system with unitary $W$.} We begin by introducing notions from the Dilation Theory. 

Dilation Theory is a branch in operator theory that seeks to embed a linear operator in another with {desirable}
properties. Given an $n\times n$ matrix $W\in \mathbb{M}_{n\times n}$ over $\mathbb{C}$ 
with  {operator} norm $\|W\|\leq 1$, \cite{Halmos1950} observed that one can embed $W$ in a unitary operator 
\[\tilde U=\begin{bmatrix}
W & D_{W^*} \\
D_W & -W^*
\end{bmatrix}\]
where $D_W=(I-W^*W)^{1/2}$ and $D_{W^*}=(I-WW^*)^{1/2}$.  {Here, $\|W\|\leq 1$ is necessary so that $I-WW^*$ and $I-W^*W$ are positive semidefinite. We notice that in this construction, however,} the information on $W$ is lost when we try to compute powers of $\tilde U$, as $\tilde U^2$ no longer has $W^2$ in the upper left block.
%corner. 
The seminal work of \cite{Nagy1953} (see also, \cite{Schaeffer1955}) proved that one can find a unitary  {operator} $U$ on an infinite-dimensional Hilbert space $\mathcal{H}$ and an isometric embedding $J:\mathbb{C}^n\to\mathcal{H}$ such that $W^k = J^* U^k J$ for all $k\in\mathbb{Z}$. 
% {I PUT $\tilde U$ ABOVE TO DISTINGUISH THE FIRST SUGGESTION FOR THE UNITARY MATRIX FROM THE ONE WE ARE GOING TO USE.}
% {PLEASE CHECK: SINCE $U$ IS UNITARY AND $S$ ISOMETRY, SHOULDN'T WE HAVE THE DETERMINANT (NORM) OF $W$ STRICTLY EQUAL TO 1 (VOLUME PRESERVING)? I CANNOT SEE HOW THIS CAN HOLD FOR A CONTRACTION?}
% {
%Take e.g. $T=1/2$, $U=\begin{bmatrix}
%    1/2 & 0 & \sqrt{3}/2 \\
%    \sqrt{3}/2 & 0 & -1/2 \\
%    0 & 1 & 0
%\end{bmatrix}$
%$S$ maps a copy of C to the first coordinate. $T=S^* U S$.
%changed $S$ to $J$
%}

The space $\mathcal{H}$ in Sz.Nagy's dilation is necessarily infinite-dimensional in general. However, if we only require $W^k = J^* U^k J$ for all $1\leq k\leq N$, a result of \cite{Egervary1954} shows that this can be achieved with the following $(N+1)n\times (N+1)n$ unitary matrix $U$:

\[U=\begin{bmatrix}
W   & 0  & 0 & \cdots & \cdots      & 0   & D_{W^*} \\
D_W & 0  & 0 & \cdots & \cdots      & 0   & -W^* \\
 0   & I &  0     & \cdots   & \cdots & 0 & 0 \\
 \vdots   & 0  & \ddots&    &   &  \vdots & \vdots\\
  \vdots   &  \vdots & &  \ddots  &   &  \vdots & \vdots\\
  \vdots   &   & &    & \ddots  &  0 & \vdots\\

 0   & \cdots  &  \cdots & \cdots &  0 & I  & 0 
\end{bmatrix}.\]
 {The embedding isometry will be now realised by an $(N+1)n \times n$ matrix $J$ over $\mathbb{C}$.
For more background on Dilation Theory we refer interested readers to \cite{PaulsenBook}}. 

 {It will often be the case that when we transform one reservoir system into another the corresponding readout mappings will be closely related to each other. In particular:}

 {
\begin{definition}
    Given two functions $h,g$ sharing the same domain $D \subset \mathbb{K}^n$, where $\mathbb{K}^n$ is a field, we say that \textbf{$g$ is $h$ with linearly transformed domain} if there exists a linear transformation over $\mathbb{K}^n$ with the corresponding matrix $A$ such that $g(\x) = h(A\x)$ for all $\x \in D$.
\end{definition}
}
 {We now demonstrate how to use the dilation technique to obtain $\epsilon$-close approximating reservoir systems with a unitary matrix $W$. }
%We begin by a definition that describes the adaptation of readout maps in the subsequent discussions.

% In Section \ref{sec:unitary} we first show that any linear reservoir system can the approximated by a linear reservoir system with unitary dynamic coupling matrix.

\begin{theorem}\label{thm.dilation} Let $R=(W,V,h)$ be a reservoir system defined by contraction $W$ with $\norm{W}=:\lambda\in (0,1)$ and satisfying the assumptions of Definition~\ref{def.lrc}. Given $\epsilon>0$, there exists a reservoir system $R'=(W', V', h')$ that is $\epsilon$-close to $R$ with $W'=\lambda U$ for a unitary $U$. Moreover, $h'$ is $h$ with linearly transformed domain.
\end{theorem}

% \dilation

% \begin{theorem}\label{thm.dilation} Let $R=(W,V,h)$ be a reservoir system defined by contraction $W$ with $\norm{W}=:\lambda\in (0,1)$. Given $\epsilon>0$, there exists a reservoir system $R'=(W', V', h')$ that is $\epsilon$-close to $R$ with $W'=\lambda U$ for a unitary $U$.
% \end{theorem}

\begin{proof} The uniform boundedness of input stream and contractiveness of $W$ imply that the state space $X\subseteq \mathbb{C}^n$ is closed and bounded, hence compact. The continuous readout map $h$ is therefore uniformly continuous on the state space $X$. 
%\comment{Add $h$ is uniformly continuous here?: Assumption 2 also implies: the state space $X$ is bounded (closed?) thus compact. So $h$ is a continuous map on a compact set hence uniformly continuous. I.e. input is uniformly bounded, the image onto the state space is compact, hence $h$ is uniformly continuous. We did make the assumption that $h$ is continuous tho, which I think is mild.}
By the uniform continuity of $h$, for {any} $\epsilon >0$, there exists $\delta>0$ such that  {for any $\x,\x' \in X$ with $\|\Bx-\Bx'\|<\delta$, we have $\|h(\Bx)-h(\Bx')\|<\epsilon$}. Let $\lambda = \|W\|$ and let $M$ denote the uniform bound of  $\{\Bc_t\}$ such that $\Bc_t\leq M$ for all $t$. Since $\lambda <1$, we can choose $N$, such that:
 {
\[
2 M \|V\| \sum_{t>N} \|W\|^t  = 2 M \|V\|  \frac{\lambda^{N+1}}{1-\lambda} < \delta.
\]
}
Let $W_1 = W/\lambda$ and $n'=(N+1) \cdot n$. 
We have $\|W_1\|=1$ and therefore by Egerv\'{a}ry's dilation, there exists a unitary $n'\times n'$ matrix $U$  such that for all $1\leq k\leq N$, we have: 
\begin{align*}
    W_1^k = J^* U^k J , 
\end{align*}
% \[W_1^k = T^* U^k T.\]
where $J:\mathbb{C}^{n} \hookrightarrow \mathbb{C}^{n'}$ is the canonical embedding 
% {[SHOULDN'T IT BE $S:\mathbb{C}^{n} \hookrightarrow \mathbb{C}^{n'}$?]}
% {Yes I think by dilation section it should be reversed?}
of $\mathbb{C}^n$ onto the first $n$-coordinates of $\mathbb{C}^{n'}$. Let $W'=\lambda U$, then it follows immediately that: 
\[
W^k = \lambda^k W_1^k = J^* \left(\lambda U\right)^k J = J^* \left(W'\right)^k J.
\]
Define an $n'\times n$ matrix by:
\begin{equation}
\label{eqn.unitary.v'}
V'=\begin{bmatrix}
V  \\
0
\end{bmatrix},
\end{equation}
 {and the map $h':\mathbb{C}^{n'} \to \mathbb{C}^d$ given by $h'=h\circ J^*$, which is:
\begin{align}
\label{eqn.unitary.h'}
h'(x_1, x_2, \cdots,x_n,\cdots,x_{n'}) &:= h(x_1, x_2, \cdots, x_n).
%&= h\circ J^*(x_1, x_2, \cdots,x_n,\cdots,x_{n'}) \nonumber 
\end{align} }

We now show that the reservoir system $R'=(W', V', h')$ is $\epsilon$-close to $R=(W,V,h)$. For any input stream $\{\Bc_t\}_{t\in\mathbb{Z}_-}$,  {the states under the reservoir system $R'$ is given by: 
\begin{align}
\label{eqn.unitary.state}
%{\Bx_t} &= \sum_{k\geq 0} W^k V \Bc_{t-k} \nonumber \\
{\Bx_t'} &= \sum_{k\geq 0} \left(W'\right)^k V' \Bc_{t-k} 
\end{align}
}
For each $k \geq 0$, we denote the upper left $n\times n$ block of $\left(W'\right)^k$ by $A_k$. In other words:
\[\left(W'\right)^k = \begin{bmatrix}
A_k & * \\
* & *
\end{bmatrix}.\]
This splits into two cases. For each $0\leq k\leq N$, we have $A_k=W^k$ by construction of $W'$. Otherwise, for $k>N$, the power $k$ is beyond the dilation power bound and we no longer have $A_k=W^k$ in general. Nevertheless, since $A_k$ is a submatrix of $W^k$, its operator norm is bounded from above:
\[\norm{A_k}\leq \norm{W^k}\leq \norm{W'}^k = \lambda^k.\]

By Equation \eqref{eqn.unitary.v'}, we have $V' \Bc_{t-k}=\begin{bmatrix}
V \Bc_{t-k} \\
0
\end{bmatrix}$ and the state ${\Bx_t'}$ of $R'$ from Equation \eqref{eqn.unitary.state} thus becomes: %construction of $V'$ in

\begin{align*}
{\Bx_t'} &= \sum_{k\geq 0} \left(W'\right)^k V' \Bc_{t-k} \\
&= \sum_{k\geq 0} \begin{bmatrix}
A_k & * \\
* & *
\end{bmatrix}  \begin{bmatrix}
V \Bc_{t-k} \\
0
\end{bmatrix}\\
&= \sum_{k=0}^N \begin{bmatrix}
W^k & * \\
* & *
\end{bmatrix}  \begin{bmatrix}
V \Bc_{t-k} \\
0
\end{bmatrix}  + \sum_{k>N} \begin{bmatrix}
A_k & * \\
* & *
\end{bmatrix}  \begin{bmatrix}
V \Bc_{t-k} \\
0
\end{bmatrix} \\
&= \begin{bmatrix}
\sum_{k=0}^N W^k V \Bc_{t-k} \\
*
\end{bmatrix} + \begin{bmatrix}
\sum_{k>N} A_k V \Bc_{t-k} \\
*
\end{bmatrix}.
\end{align*}

Let $J^*({\Bx_t'})$ be the first $n$-coordinates of ${\Bx_t'}$. 
% {[SEE THE COMMENT ABOUT $S$ ABOVE]}
We have 
%\comment{ADDED UNDERBRACE FOR CLARIFY, NOT SURE IF NEEDED}
\[J^*({\Bx_t'})= \sum_{k=0}^N W^k V \Bc_{t-k}+\sum_{k>N} A_k V \Bc_{t-k}.\]
Comparing  {this with the state generated by the reservoir system $R$, which is given by:} 
\[{\Bx_t} = \sum_{k\geq 0} W^k V \Bc_{t-k} = \sum_{k=0}^N W^k V \Bc_{t-k} + \sum_{k>N} W^k V \Bc_{t-k},\]
it follows immediately that:
\[
\norm{J^*({\Bx_t'})-{\Bx_t}} = \norm{0+\sum_{k>N} \left(A_k - W^k\right) V \Bc_{t - k}} \leq 
\sum_{k>N} \left(\norm{A_k}+\norm{W^k}\right) \norm{V} M.
\]
Notice we have $\|W^k\|\leq \|W\|^k=\lambda^k$ and we also showed $\norm{A_k}\leq \lambda^k$, and therefore:
% Since $A_k$ is a corner of $W'^k$ \comment{and since $S$ is co-isometry (or the canonical embedding) so it preserves operator norm? for $k>N$ do we have $A_k = S^* W'^k S$? Re: for $k>N$, $A_k = S^* W'^k S$, and we know $\|S\|=1$ since it is an co-isometry.}, we have \comment{$\|W'^k\| = \|W^k\|$ still true for $k > N$? is this just by construction that $W' = \lambda U$? Re: you're right, $\|W'^k\|$ may not be $\|W^k\|$. We are simply using $\|W'^k\|=\lambda^k \|U^k\|=\lambda^k$ since $U$ is unitary. Add a couple more sentences for the CS community ? such as: The operator norm of the upper RHS corner $A_k$ for $k>N$ is bounded above (why) by the operator norm of unitary $W'^k$, hence..}
% \[\|A_k\| \leq \|W'^k\| = \lambda^k\]
\[\|J^*({\Bx_t'})-{\Bx_t}\|\leq  \sum_{k>N} 2\lambda^k \|V\| M < \delta. \]
By Equation \eqref{eqn.unitary.h'} $h'({\Bx_t})=h(J^*({\Bx_t}))$ and by uniform continuity of $h$ we have:
\[\|{\By_t}-{\By_t'}\| = \|h({\Bx_t}) - h(J^*({\Bx_t}))\|<\epsilon. \]
This finishes the proof. 
\end{proof}

\section{From Unitary to Permutation State Coupling}
\label{sec:permutation_univ}

In this section, we show that for any reservoir system with unitary state coupling, we can construct an equivalent reservoir system with  {a full cyclic coupling} of state units weighted by a single common connection weight value. To that end, we first show that matrix similarity of dynamical coupling implies reservoir equivalence.  
%\commentgreen{Matrix similarity of dynamical coupling implies reservoir equiv. }

% \begin{proposition} \comment{Marketing ver.}
% Let $R=(W,V,h)$ be a reservoir system defined by dynamical coupling matrix $W$. Let $S$ be an invertible matrix such that $W' := S^{-1} W S$ is a matrix similar to $W$ Then there exists an input map $V'$ and readout $h'$ such that reservoir system $R':=(W', V', h')$, that is equivalent to $R$. 
% \end{proposition}

\begin{proposition}\label{prop.similar} Let $W$ be a contraction and let $R=(W,V,h)$ denote the corresponding reservoir system. Suppose $S$ is an invertible matrix such that $W'=S^{-1} W S$ and $\|W'\|<1$. Then there exists a reservoir system $R'=(W', V', h')$ that is equivalent to $R$. 
\end{proposition}

% {IT WOULD BE BETTER TO RESERVE $S$ FOR THE EMBEDDING IN DILATION. CAN WE USE A DIFFERENT SYMBOL FOR THIS MATRIX?} \comment{Since $S$ here serves the same purpose as above should we keep the same symbol?}
% {I see Peter's point. Changed previous $S$ in dilation to $J$}
\begin{proof} 
 {
Let $V'= S^{-1} V$ and $h'(\Bx)=h(S\Bx)$. Given inputs $\left\{\Bc_t\right\}_{t\in \mathbb{Z}_-}$, we have the solutions to the systems $R=(W,V,h)$ given by
\begin{align*}
        {\By_t} &= h\left(\sum_{n\geq 0} W^n V \Bc_{t-n}\right). 
        %\ \ \hbox{and}\\
        %{\By_t'} &= h'\left(\sum_{n\geq 0} \left(W'\right)^n V' \Bc_{t-n}\right),
\end{align*}
Similarly the solutions to $R'=(W',V',h')$ is given by:
%respectively.}
%Therefore, we have:
\begin{align*}
    {\By_t'} 
    &= h'\left(\sum_{n\geq 0} \left(W'\right)^n V' \Bc_{t-n}\right) \\
    &= h\left(S\cdot \sum_{n\geq 0} (S^{-1} W S)^n {S^{-1}} V \Bc_{t-n}\right) \\
    %&= h\left(S\cdot \sum_{n\geq 0} S^{-1} W^n S \cdot {S^{-1}} V \Bc_{t-n}\right) \\
    &= h\left(\sum_{n\geq 0} W^n V \Bc_{t-n}\right) =  {\By_t}.
\end{align*}
This proves that the two systems are equivalent. 
}
\end{proof}

We now show that for any given unitary state coupling we can always find a full-cycle permutation that is close to it to arbitrary precision. This is done by perturbing a given unitary matrix to one that is unitarily equivalent to a cyclic permutation. It is well known that the eigenvalues of unitary matrices lie on the unit circle $\mathbb{T}$ in $\mathbb{C}$, and the eigenvalues of cyclic permutations are the roots of unities in $\mathbb{C}$. Given a unitary matrix $U$, we can therefore first perturb its eigenvalues to a subset of eigenvalues of a  {cyclic} permutation matrix.  {The remaining roots of unity (eigenvalues of the cyclic permutation) not covered by the previous operation 
%from the eigenvalues of the cyclic permutation 
are then filled in using direct sum with the diagonal matrix} consisting of the missing eigenvalues. 
%The above discussion is summarized in the following results.

% \footnote{ {removed the perturbation to permutation since a full-cycle is a permutation already}}

\begin{proposition}\label{prop.perturb.unitary} Let $U$ be an $n\times n$ unitary matrix and $\delta>0$ be an arbitrarily small positive number. There exists an $n_1\times n_1$ matrix $A$ with $n_1 > n$ that is unitarily equivalent to a full-cycle permutation, and an $(n_1-n)\times (n_1-n)$ diagonal matrix $D$ such that:
\[
\norm{A-\begin{bmatrix} U & 0 \\ 0 & D\end{bmatrix}}<\delta .
\] 
\end{proposition}

\begin{proof} 
Given an integer $\ell \ge 1$, the complete set of $\ell^{th}$ roots of unity, denoted by $\mathcal{R}_\ell=\{e^{i\frac{2k\pi}{\ell}}: 0\leq k\leq \ell-1\} \subset \mathbb{T}$, is a collection of uniformly positioned points along the complex circle $\mathbb{T}$. 
% {[WE ARE USING $m$ TO DENOTE INPUT DIMENSIONALITY, SHOULD USE A DIFFERENT SYMBOL FOR CONSISTENCY]} % {Done. used $\ell$ instead.}
It is well known from elementary matrix analysis that an $\ell\times \ell$ full-cycle permutation matrix is unitarily equivalent to a diagonal matrix whose diagonal entries consist of the complete set  {$\mathcal{R}_\ell$} of $\ell^{th}$ roots of unity. 
Therefore, an $\ell\times \ell$ matrix $X$ is unitarily equivalent to a full-cycle permutation if and only if its eigenvalues are precisely $\mathcal{R}_\ell$.

Let $U$ be a fixed $n\times n$ unitary matrix $U$ and denote its 
eigenvalues by $\{\omega_1, \cdots,\omega_n\}$. Since $U$ is unitary, $|\omega_j|=1$ and thus $w_j \in \mathbb{T}$ for all $ j = 1,\ldots, n$.
We write $w_j=e^{2\pi i a_j}$ for $a_j\in [0,1)$. For any $\delta>0$, pick an integer $\ell_0 > 0$ such that 
\[\left|1-e^{\frac{\pi i}{\ell_0}}\right|<\delta.\]
We claim that there exists distinct integers $b_1, \cdots, b_n$, $0\leq b_j < \ell_0 \cdot n$ such that 
\[\left|a_j - \frac{b_j}{\ell_0 \cdot n}\right|\leq \frac{1}{2\ell_0}\]
Indeed, for each $j$, 
%  {[YOU MEAN: for each $j$?]}
the range for allowable $b_j$ is given by 
\[a_j \ell_0 n - \frac{n}{2} < b_j < a_j \ell_0 n + \frac{n}{2}.\]
This is an interval of length $n$, and thus there are precisely $n$ possible choices of $b_j$. Since we need to pick $n$ of $b_j$'s, it is always possible to  {make $b_1, \cdots, b_n$ 
%$b_j$ 
distinct}. 
 {We thus obtain distinct integers
%Now suppose 
$\left\{b_j\right\}_{j=1}^n$ 
%are distinct integers 
satisfying} 
\[\left|a_j - \frac{b_j}{\ell_0 \cdot n}\right|<\frac{1}{2\ell_0}.\]
We have that: % divide the second element out, since that is in $\mathbb{T}$ it has norm 1 and doesn't change. Then swap them in abs to get the second term. In this case the first equal should be leq? -- if so it doesn't affect the result anyway
\[\left|\omega_j - e^{2\pi i \frac{b_j}{\ell_0 \cdot n}}\right| = \left|1 - e^{2\pi i \left|a_j - \frac{b_j}{\ell_0 \cdot n}\right|}\right| \leq  \left|1 - e^{2\pi i \frac{1}{2\ell_0}}\right| < \delta.\]

Let $n_1=\ell_0 \cdot n$. Consider two diagonal matrices $D_U :=\operatorname{diag}\left\{\omega_1, \cdots, \omega_n\right\}$ and $D_1 := \operatorname{diag}\left\{e^{2\pi i \frac{b_1}{n_1}}, \cdots, e^{2\pi i \frac{b_n}{n_1}}\right\}$. Then:
\begin{align*}
    \norm{D_U-D_1} = \max_{j \in\left\{ 1,\ldots, n\right\}} \left|\omega_j - e^{2\pi i\frac{b_j}{n_1}}\right| < \delta. 
\end{align*}

%Therefore, we can pick $n$ \textbf{distinct} integers $\{k_j:0\leq k_j\leq N-1\}_{j=1}^n$ such that:
%\[|w_j - e^{i\frac{2\pi}{N} \cdot k_j}|<\delta.\]
%Let $D_U :=\operatorname{diag}\left\{\omega_1, \cdots, \omega_n\right\}$ and $D_1 := \operatorname{diag}\left\{e^{i\frac{2 \pi}{N} \cdot k_1}, \cdots, e^{i\frac{2\pi}{N}\cdot k_n }\right\}$. Then:
%\begin{align*}
%    \|D_U-D_1\| = \max_{j \in\left\{ 1,\ldots, n\right\}} |\omega_j - e^{i\frac{2 \pi}{N}\cdot k_j }| < \delta. 
%\end{align*}
Let $S$ be unitary such that $U = S^*D_U S$. Let $U_1 : =S^*D_1 S$, 
% {[WE USE $V$ TO DENOTE INPUT-TO-STATE COUPLING], PLEASE USE A DIFFERENT SYMBOL]}
% {Done, changed to $U_1$. Should be $U = S^*D_U S$ and $U_1 : =S^*D_1 S$ instead?}
% {Re Simon: yup, changed}
% {Do we change this $S$ to $J$ as well?}
then by construction we have $\|U-U_1\|<\delta$. Moreover, the set of eigenvalues of $U_1$, given by $\sigma_{U_1}=\left\{e^{\frac{2\pi i b_j }{n_1}}\right\}_{j=1}^{n} \subset \mathcal{R}_{n_1}$, consists of distinct $n_1^{th}$ roots of unity. Whilst this is not a \textit{complete} set of roots of unity $\mathcal{R}_{n_1}$, we can complete the set by filling in the missing ones. More precisely let $D$ be the diagonal matrix consisting of all the missing $n_1$-th roots of unity,
 {$D = diag\{ \mathcal{R}_{n_1} \setminus \sigma_{U_1}\}$, where}
 {
\[
\mathcal{R}_{n_1} \setminus \sigma_{U_1}:= \left\{e^{\frac{2\pi i b}{n_1}}: 1 \le b \le n_1,  b\neq b_j \right\}.
\]
} 
{Then $D \in \mathbb{C}_{d\times d}$ with $d = \left|\mathcal{R}_{n_1}\right| - \left|\sigma_{U_1}\right| = n_1 - n$. The block diagonal matrix $A:=\begin{bmatrix} U_1 & 0 \\ 0 & D\end{bmatrix} \in \mathbb{C}_{n_1\times n_1}$ is unitarily equivalent to a cyclic permutation as its eigenvalues form a complete set of roots of unity $\mathcal{R}_{n_1}$. 
%  {[SHOULD BE $\mathcal{R}_{n_1}$?]}
Finally, by the construction of $A$: 
\[
\norm{A-\begin{bmatrix} U & 0 \\ 0 & D\end{bmatrix}} = \norm{ \begin{bmatrix} U_1-U & 0 \\ 0 & 0\end{bmatrix}} = \norm{U_1-U}  <\delta,
\]
as desired. 
}
\end{proof}

% \topermutation
% In Section \ref{sec:permutation_univ} we show that any linear reservoir system with a unitary dynamic coupling can be approximated by a linear reservoir system with full-cycle permutation state coupling.
% {In the following we are not using $\Bx^{(i)}$ like in SMCR, shall we use a different notation or is this fine?}
% {[I THINK IT IS BETTER TO BE CONSISTENT THROUGHOUT THE PAPER, BUT IN THIS CASE HOPEFULLY THERE WILL BE NO CONFUSION WITH MULTI-CYCLE]}

\begin{theorem}\label{thm.to.permutation} Let $U$ be an $n\times n$ unitary matrix and $W = \lambda U$ with $\lambda \in \left(0,1\right)$. Let $R=(W,V,h)$ be a reservoir system that satisfies the assumptions of Definition~\ref{def.lrc} with state coupling $W$. For any $\epsilon>0$, there exists a reservoir system $R_c=(W_c, V_c, h_c)$ that is $\epsilon$-close to $R$ such that:
\begin{enumerate}
\item $W_c$ is a contractive full-cycle permutation with $\|W_c\|=\|W\| = \lambda \in \left(0,1\right)$, and
\item $h_c$ is $h$ with linearly transformed domain. 
% \item $h_c$ the composition of $h$ with a linear map. 
\end{enumerate}
\end{theorem}

% \begin{theorem}\label{thm.to.permutation} Let $U$ be an $n\times n$ unitary matrix and $W = \lambda U$ with $\lambda \in \left(0,1\right)$. Let $R=(W,V,h)$ be a reservoir system defined by $W$. For any $\epsilon>0$, there exists a reservoir system $R_c=(W_c, V_c, h_c)$ that is $\epsilon$-close to $R$ with:
% \begin{enumerate}
% \item $W_c$ is a contractive full-cycle permutation with $\|W_c\|=\|W\| = \lambda \in \left(0,1\right)$, and
% \item $h_c$ the composition of $h$ with a linear map. 
% \end{enumerate}
% \end{theorem}

\begin{proof} Let $\epsilon >0$ be arbitrary. By the proof of Theorem \ref{thm.dilation},
the state space $X$ is compact and we can choose $\delta$ such that $\|\Bx-\Bx'\|<\delta$ implies $\|h(\Bx)-h(\Bx')\|<\epsilon$. Let $M:=\sup \|\Bc_t\|<\infty$, since $\lambda < 1$ we can pick $N>0$ such that 
\begin{align}
\label{eqn.halfdelta_end}
2 M \|V\|  \sum_{k>N} \lambda^k  < \frac{\delta}{2}.
\end{align}
Once we fix such an $N$, pick $\delta_0>0$ such that
\begin{align}
\label{eqn.halfdelta_front}
M \|V\| \sum_{k=0}^N ((\lambda+\delta_0)^k-\lambda^k)  < \frac{\delta}{2}.
\end{align}
Such a $\delta_0$ exists because the left-hand side is a finite sum that is continuous in $\delta_0$ and tends to $0$ as $\delta_0\to 0$.
According to Proposition~\ref{prop.perturb.unitary}, there exists a $n_1\times n_1$ matrix $A$ that is unitarily equivalent to a full-cycle permutation and a diagonal matrix $D$ such that 
\[
\norm{A-\begin{bmatrix} U & 0 \\ 0 & D\end{bmatrix}}<\frac{1}{\lambda}\min\{\delta,\delta_0\}.
\]

Let $Q_n:\mathbb{C}^{n_1} \hookrightarrow \mathbb{C}^n$ be the canonical projection onto the first $n$ coordinates. Consider the reservoir systems $R_0 :=(W_0, V_0, h_0)$ and $R_1 :=(W_1, V_0, h_0)$ defined by the following:
\begin{align*}
    W_0 = \lambda \begin{bmatrix} U & 0 \\ 0 & D\end{bmatrix} , & \quad V_0 = \begin{bmatrix} V \\ 0 \end{bmatrix} \\
    W_1 = \lambda A,&\quad  h_0(\Bx) = h(Q_n(\Bx)).
\end{align*}
Notice that the choice of $A$ ensures that $\norm{W_1 - W_0} <  \min\{\delta, \delta_0\}$.

The rest of the proof is outlined as follows: We first show that $R_0$ is equivalent to $R$, and then prove that $R_1$ is $\epsilon$-close to $R_0$. By Proposition \ref{prop.perturb.unitary}, $A$ is unitarily equivalent to a full-cycle permutation matrix, and the desired results follow from Proposition \ref{prop.similar}. We now flesh out the above outline: 

% \begin{enumerate}
%   \setcounter{enumi}{0}
%   \item We first establish that $R_0$ is equivalent to $R$. 
% \end{enumerate}

We first establish that $R_0$ is equivalent to $R$. 
For any input stream $\{\Bc_t\}_{t\in\mathbb{Z}_-}$, the solution to $R_0$ is given by
\begin{align*}
    {\By}_t^{(0)} &= h_0\left(\sum_{k\geq 0} W_0^k V_0 \Bc_{t-k}\right) \\
    &= h\left(Q_n \left(\sum_{k\geq 0} \begin{bmatrix} (\lambda U)^k & 0 \\ 0 & (\lambda D)^k \end{bmatrix} \begin{bmatrix} V \\ 0 \end{bmatrix} \Bc_{t-k}\right)\right) \\
    &= h\left(Q_n \left(\begin{bmatrix} \sum_{k\geq 0} W^k V \Bc_{t-k} \\0 \end{bmatrix}\right)\right) \\
    &= h\left(\sum_{k\geq 0} W^k V \Bc_{t-k}\right).
\end{align*}
This is precisely the solution to $R$. 

We now show that $R_1$ is $\epsilon$-close to $R_0$. 

First, we observe that since $Q_n$ is a projection onto the first $n$-coordinates, it has operator norm $\norm{Q_n}=1$
% {[NOT COMPLETELY TRUE, IT IS NOT CONTRACTIVE FOR STATES WITH THE LAST $n_1 - n$ COORDINATES EQUAL TO 0. THERE IT IS IDENTITY]}
% {Maybe this is confusion over the word contractive. Here I mean the operator norm $\leq 1$. changed the wording here, we only need its operator norm.}
and thus  {whenever  $\norm{\Bx-\Bx'}<\delta$, $\Bx,\Bx'\in\mathbb{C}^{n_1}$, we have} $\norm{Q_n \Bx- Q_n \Bx'}<\delta$ and thus $\|h(Q_n \Bx)-h(Q_n \Bx')\|<\epsilon$. Therefore it suffices to prove that for any input $\{\Bc_t\}$, the solution to $R_0$, given by
\[{\Bx}_t^{(0)} = \sum_{k\geq 0} W_0^k V_0 \Bc_{t-k},\]
is within $\delta$ to the solution to $R_1$, given by
\[{\Bx}_t^{(1)} = \sum_{k\geq 0} W_1^k V_0 \Bc_{t-k}.\]
By construction $V_0=\begin{bmatrix} V \\ 0\end{bmatrix}$ has $\|V_0\|=\|V\|$, hence:
\begin{align}
\label{eqn.state_r0r1_diff}
 \norm{{\Bx}_t^{(0)} - {\Bx}_t^{(1)}} &= \norm{\sum_{k\geq 0} (W_0^k-W_1^k) V_0 \Bc_{t-k}} \nonumber \\
&\leq \sum_{k\geq 0}  \norm{(W_0^k-W_1'^k)} \norm{V_0} M \nonumber \\
&= \sum_{k=0}^N  \norm{\left(W_0^k-W_1^k\right)} \norm{V} M + \sum_{j>N}  \norm{\left(W_0^k-W_1^k\right)} \norm{V} M.
\end{align}

Consider  {the matrix} $\Delta=W_0-W_1$, we then have $\|\Delta\|<\delta_0$ and for each $0\leq j\leq N$, $W_0^j-W_1^j=(W_1+\Delta)^j - W_1^j$. Expanding $(W_1+\Delta)^j$, we get a summation of $2^j$ terms of the form $\prod_{i=1}^j X_i$, where each $X_i=W_1$ or $\Delta$.  For $s = 0,\ldots , j$, each of the $2^j$ terms has norm $\norm{\prod_{i=1}^j X_i}\leq \|W_1\|^{j-s} \|\Delta\|^{s}$ if there are $s$ copies of $\Delta$ among $X_i$. 
%$W_0^j - W_1^j$ is thus obtained by 
 {Removing the term $W_1^j$ from $(W_1+\Delta)^j$ results in all the remaining terms containing at least one copy of $\Delta$.} We thus arrive at: 

\begin{align}
    \norm{W_0^j - W_1^j} &= \norm{\left(W_1+\Delta\right)^j - W_1^j} 
    \nonumber \\
    &\leq \sum_{s=1}^j \binom{j}{s} \norm{W_1}^{j-s} \norm{\Delta}^{s} 
    \nonumber \\
   &\leq \sum_{s=1}^j \binom{j}{s} \lambda^{j-s} \delta_0^{s} = (\lambda+\delta_0)^j - \lambda^j.
\label{eq:bound}
\end{align}
Combining the above with Equation \eqref{eqn.halfdelta_front}, we have:
\[
M \|V\| \sum_{j=0}^N  \|(W_0^j-W_1^j)\| \leq 
M \|V\| \sum_{j=0}^N ((\lambda+\delta_0)^j - \lambda^j) < \frac{\delta}{2}. 
\]
On the other hand by Equation \eqref{eqn.halfdelta_end}, we obtain:
\begin{eqnarray*}
M \|V\| \sum_{j>N}  \|(W_0^j-W_1^j)\|  
&\leq& 
M \|V\| \sum_{j>N}  (\|W_0\|^j+\|W_1\|^j)\\
&\leq& 2 M \|V\| \sum_{j>N}  \lambda^j\\
&<&\frac{\delta}{2}.
\end{eqnarray*}
With the two inequalities above, Equation \eqref{eqn.state_r0r1_diff} thus becomes:
\[ \|{\Bx}_t^{(0)} - {\Bx}_t^{(1)} \| <\delta.\]
Uniform continuity of $h$ implies $\norm{h({\Bx}_t^{(0)}) - h({\Bx}_t^{(1)})}<\epsilon$, proving $R_1$ is $\epsilon$-close to $R_0$. 

Finally, by Proposition \ref{prop.perturb.unitary} $A$ is unitarily equivalent to a full-cycle permutation matrix $P$, i.e. there exists unitary matrix $S$ such that $S^* A S=P$. By Proposition~\ref{prop.similar}, we obtain a reservoir system $R_c=(W_c, V_c, h_c)$ with $W_c=\lambda P$, such that $R_c$ is equivalent to $R_1$, which is in turn $\epsilon$-close to $R_0$. Since the original reservoir system $R$ is equivalent to $R_0$, $R$ is therefore $\epsilon$-close to $R_c$, as desired.
\end{proof}

\section{Universality of \PermutationReservoir, \texorpdfstring{$\mathbb{C}$-SCR} ,, and Twin SCRs}\label{sec:mainresults}

We are now ready to prove the main results: the universality of three distinctive linear reservoir systems: \PermutationReservoir, Complex Simple Cycle Reservoir and a Twin Simple Cycle Reservoir. We begin by showing  
% a stronger version of Theorem \ref{thm.to.permutation}: 
that any linear reservoir system $R$ is $\epsilon$-close to a \PermutationReservoir $R' = (W', V', h')$. That is, the approximant system $R'$ has both a contractive coupling matrix $W'$ in block-diagonal form with identical contractive full-cycle permutation blocks, and an input map (input-to-state coupling) $V'$ whose entries are all $\pm 1$.

%\comment{since the proof is quite split I felt we should split into 3 main theorems... instead of 2. See tikz}:

\begin{theorem}\label{thm.main.pr} 
For any reservoir system $R=(W,V,h)$ that satisfies the assumptions of Definition~\ref{def.lrc}  and any $\epsilon>0$, there exists a \PermutationReservoir \ $R'=(W',V', h')$ that is $\epsilon$-close to $R$. 
Moreover, $\norm{W}=\norm{W'}$ and  $h'$ is $h$ with linearly transformed domain. 
% $h'$ and $h$ have the same function form \footnote{Not sure if this is the correct term. Want to describe here  composed with a linear transformation.}
\end{theorem} 
% of dimensions $(n,m,d)$  of dimension $(n',m,d)$  

% If we want to put the statement directly here
% \comment{so the permutation case is done implicitly?, i suppose this implies we don't need the previous proof? or do we split into new theorem?}
% \begin{remark}
%     In  we have shown that any  reservoir system $R'=(W', V', h')$ where $W'$ is a  It remains to show that one can make the entries in the matrix $V$ to be all $\pm 1$. 
% % {$\pm i$?} No, W will be an assembly of full-cycle, V can be real.
% \end{remark}

\begin{proof}
%[Proof of Theorem~\ref{thm.main.pr}]
Consider a reservoir system $R=(W,V,h)$ with dimensions $(n,m,d)$. By Theorem~\ref{thm.dilation} and Theorem~\ref{thm.to.permutation}, we may assume without loss of generality that the coupling matrix $W$ is a contractive full-cycle permutation, that is, $W=\lambda P$ for some full-cycle permutation $P$ and $\lambda=\|W\|<1$. 
%\comment{in this case the dimension would be no longer $n$? I suppose the new dimension is a function of this new $n$ shouldn't be an issue? 
%Re We are assuming WLOG W is already a full-cycle.}

     {Let $\{E_i\}_{i=1}^{nm}$ be a vector space basis for $\mathbb{M}_{n\times m}$ such that each $E_i\in \mathbb{M}_{n\times m}(\{1,-1\})$.} We have $V=\sum a_i E_i$ for some constants $a_i\in \mathbb{C}$. Consider the permutation reservoir $R'=(W',V',h')$ of dimension $(n^2 m , m,d)$ defined by the following:

\begin{align*}
    W' := \begin{bmatrix} W & & \\
    & \ddots & \\
    & & W \end{bmatrix}, &\quad V':=\begin{bmatrix} E_1 \\ E_2 \\ \vdots \\ E_{nm}\end{bmatrix} \\
h'({\Bx}^{(1)}, \cdots, {\Bx}^{(nm)}) &= h\left( \sum_{i=1}^{nm} a_i \Bx^{(i)}\right).
\end{align*}
By construction entries in $V'$ are all $\pm 1$ and $W'$ is block-diagonal with $nm$ identical blocks of $W$. $W'$ is therefore a contractive \textit{permutation}; We note that $W'$ is no longer a full-cycle permutation from this step. 
%  removed contractive, since it was a bit misleading
%\comment{note: not a full-cycle permutation anymore? so $R'$ isn't SCR then}. 

 Given any input stream $\{c_t\}_{t\in\mathbb{Z}_-}$, the solution to $R'$ is given by:
\[{\By'_t} = h'\left(\sum_{n\geq 0} \left(W'\right)^n V' \Bc_{t-n}\right), \]
where by construction:
\[\sum_{n\geq 0} \left(W'\right)^n V' \Bc_{t-n} = \begin{bmatrix} \sum_{n\geq 0} W^n E_1 \Bc_{t-n} \\ \sum_{n\geq 0} W^n E_e \Bc_{t-n} \\ \vdots \\ \sum_{n\geq 0} W^n E_{nm} \Bc_{t-n}\end{bmatrix}. \]
Therefore,
\begin{align*}
h'\left(\sum_{n\geq 0} \left(W'\right)^n V' \Bc_{t-n}\right) 
&= h\left(\sum_{i=1}^{nm} a_i \sum_{n\geq 0} W^n E_{nm} \Bc_{t-n}\right) \\
&= h\left(\sum_{n\geq 0} W^n\left (\sum_{i=1}^{nm} a_i E_{nm}\right) \Bc_{t-n}\right) \\
&= h\left(\sum_{n\geq 0} W^n V \Bc_{t-n}\right)
\end{align*}
This is precisely the solution to $R=(W,V,h)$. 
\end{proof}

We have now shown that any reservoir system $R$ defined by a  {full-cycle} permutation coupling and arbitrary input map is equivalent to \PermutationReservoir \ in the sense of Definition \ref{def.rc}. By Theorem \ref{thm.dilation} and Theorem~\ref{thm.to.permutation},  {any reservoir system is $\epsilon$-close to a Simple Multi-cycle Reservoir}. The  {remaining two cases of Simple Cycle Reservoir structures require} a more careful construction.

We first  {show that a full-cycle block arrangement of individual full-cycle permutation blocks can be under some conditions rearranged into a larger full-cycle permutation matrix.}

\begin{lemma}\label{lm.fullcycleW} Let $n,k$ be two natural numbers such that $\gcd(n,k)=1$. Let $P$ be an $n\times n$ full-cycle permutation. Consider the $nk\times nk$ matrix:
\[P_1=\begin{bmatrix} 
0   & 0  & 0 & & \hdots      & 0   & P \\
P & 0  & 0 & & \hdots      & 0   & 0 \\
 0   & P &  0     & \hdots   & & 0 & 0 \\
 \vdots   &   & \ddots&    &   &  \vdots & \vdots\\
 0   & \hdots  &  &  &   & P  & 0 
\end{bmatrix}.\]
Then $P_1$ is a full-cycle permutation. 
\end{lemma}

\begin{proof}
By construction, $P_1$ is a permutation matrix  since each row and each column has all $0$ except one entry of $1$. Denote the canonical basis in $\mathbb{R}^{nk}$ by 
 {
$\mathcal{E} := \left\{\e_1, \e_2, \cdots, \e_{nk}\right\}$. }
%Note that the basis is indexed in $\mod \left(nk\right)$, which we omit for simplicity in notation. 

For each $i = 0,\ldots, k-1$, consider the $i^{th}$ block $\mathcal{E}_i:= \left\{\e_{in+j}\right\}_{j=1}^n \subset \mathcal{E}$. Then by construction $P_1$ maps the $\mathcal{E}_i$ to $\mathcal{E}_{i+1}$ , i.e. 
\[P_1: \left\{\e_{in+j}\right\}_{j=1}^n  \mapsto \left\{\e_{in+j+n}\right\}_{j=1}^n  = \left\{\e_{\left(i+1\right) n+j}\right\}_{j=1}^n\] 
%$\{e_{in+1}, e_{in+2}, \cdots, e_{in+n}\}$ to  $\{e_{(i+1)n+1}, e_{(i+1)n+2}, \cdots, e_{(i+1)n+n}\}$. 
Consider the $0^{th}$ block $\mathcal{E}_0 = \left\{\e_1,\ldots,\e_n\right\}$ and $\e_1$'s orbit under $P_1$ denoted by $P_1^s \e_1$, then:
\begin{align*}
P_1^s \e_1 \in \mathcal{E}_0 &\Leftrightarrow \left(1+ sn \right) \mod (nk) \in \left\{1,\ldots,n\right\} \\
	      &\Leftrightarrow \exists \alpha \in \mathbb{Z} \text{ such that } s = \alpha k \\
	      &\Leftrightarrow k \mid s.
\end{align*}

When $k \mid s$, $P_1^s \e_1 = P^s \e_1 \in \mathcal{E}_0$. Since $P$ is a full-cycle permutation of dimension $n$, for each $1\leq i\leq n$, $P^s \e_i=\e_i$ if and only if $n\mid s$. Therefore, $\e_1=P_1^s \e_1 = P^s \e_1$ implies $n\mid s$. 
%  {By a similar argument as above? Re: I think we don't need to say that. Also first line is pedantically $Q_n^*P^s Q_n(e_1)$ right? Re: technically yes, but it is clear in the context I feel. }
% \begin{align*}
% e_1 = P^s e_1=e_{i+s \mod n} &\Leftrightarrow i+s \mod n = i\\
%  &\Leftrightarrow \exists \beta \in \mathbb{Z} \text{ such that } s = \beta n \\
% &\Leftrightarrow n \mid s. 
% \end{align*}

Combining the above with the assumption that $\gcd(k,n)=1$, we have that $P_1^s \e_1=\e_1$ if and only if $nk \mid s$. The permutation $P_1$ thus contains a cycle whose length is at least $nk$. Hence, $P_1$ is a full-cycle permutation. 
\end{proof}

\begin{example}  {We emphasis that the condition $\gcd(n,k)=1$ is crucial in Lemma~\ref{lm.fullcycleW}.} Consider a simple example where $n=2$ and $k=3$. Let $P$ be the matrix for cyclic permutation $(1,2)$,
    \[P=\begin{bmatrix} 0 & 1 \\ 1 & 0\end{bmatrix}.\]
    From our construction, the matrix $P_1$ is
    \[P_1=\begin{bmatrix}
    &&&&&1 \\
    &&&&1&\\
    &1&&&&\\
    1&&&&&\\
    &&&1&&\\
    &&1&&&
    \end{bmatrix}\]
    One can check that $P_1$ corresponds to the cyclic permutation $(1,4,5,2,3,6)$.  If we picked $k=2$, then 
    \[P_1 = \begin{bmatrix}
    &&&1 \\
    &&1& \\
    &1&& \\
    1&&&
    \end{bmatrix},\]
    which is not a full-cycle permutation. 
\end{example}

\begin{lemma}\label{lm.simpleV} For any $n\times m$ real matrix $V$ and $\delta>0$, there exists $k$ matrices $\{F_1, \cdots, F_k\} \subset \mathbb{M}_{n \times m}\left(\left\{-1,1\right\}\right)$ and a constant integer $N>0$ such that:
\[\norm{V-\frac{1}{N} \sum_{j=1}^k F_j} < \delta\]
Moreover, $k$ can be chosen such that $\gcd(k,n)=1$. 
\end{lemma}

\begin{proof} Let $E:=\{E_i\}_{i=1}^{nm}\subset  \mathbb{M}_{n \times m}\left(\left\{-1,1\right\}\right)$ be a vector space basis for $\mathbb{M}_{n\times m}(\mathbb{R})$ such that entries of each $E_i$ are $\pm 1$. There exists constants $a_i\in\mathbb{R}$ such that
\[V=\sum_{i=1}^{nm} a_i E_i. \] 

Let $L:=\max\{\|E_i\|: 1\leq i\leq nm\}$. 
Pick an integer $N$ large enough such that $\frac{nm}{N}<\frac{\delta}{2L}$. Since $m>1$, we also have $\frac{nL}{N}<\frac{\delta}{2}$.
Pick a set of integers $B:= \{b_i: 1\leq i\leq nm\}$ such that: 
\[\left|a_i - \frac{b_i}{N}\right| < \frac{\delta}{2Lnm}.\]
We can always find such $b_i$ since allowable values of $b_i$ lie inside an interval of length $\frac{N\delta}{Lnm} > 2$ centred at $N a_i \in \mathbb{R}$. 

Let $k_0=\sum_{i=1}^{nm} |b_i|$ and pick the smallest integer $k\geq k_0$ such that $\gcd(k,n)=1$. It is clear that $k<k_0+n$. 

 {Now we pick $k$ matrices $\{F_1, \cdots, F_k\} \subset E$ as follows}:
\begin{enumerate}
    \item For each $1\leq i\leq nm$, $\operatorname{sgn}(b_i)E_i$ is a matrix whose entries are $\pm 1$. For each $i$ we then take $|b_i|$-copies of $\operatorname{sgn}(b_i)E_i$.

    Repeating the process across all $i \in \left\{1,\ldots,nm\right\}$, we obtain a total number of $k_0 = \sum_{i=1}^{nm} |b_i|$ matrices whose entries are all $\pm 1$ (namely $\operatorname{sgn}(b_i)E_i$ for each $b_i \in B$). We label these matrices as $F_1, \dots, F_{k_0}$. 
    % {I changed this a bit, please see if it makes sense.}
    \item Pick $k-k_0$ copies of a single basis matrix $E_j$. {Without loss of generality, we may choose $j = 1$.} 
    These will be labelled as $F_{k_0+1}$ through $F_k$. 
\end{enumerate}

Now by construction:
\begin{align*}
    \norm{V-\frac{1}{N} \sum_{j=1}^k F_j} 
    &= \norm{\sum_{i=1}^{nm} a_i E_i - \frac{1}{N} \sum_{j=1}^k F_j} \\
    &\leq \norm{\sum_{i=1}^{nm} a_i E_i - \frac{1}{N} \sum_{j=1}^{k_0} F_j} + \norm{\frac{1}{N} \sum_{j=k_0+1}^k F_j} \\
    &= \norm{\sum_{i=1}^{nm} \left(a_i - \frac{|b_i| \operatorname{sgn}(b_i)}{N}\right) E_i} + \norm{\frac{1}{N} \sum_{j=k_0+1}^k F_j} \\
    &\leq \sum_{i=1}^{nm} \left|a_i - \frac{b_i}{N}\right| \norm{E_i} + \frac{1}{N} \sum_{j=k_0+1}^k \norm{F_j} \\
    &< nm \times \frac{\delta}{2Snm} \times L + \frac{|k-k_0|}{N} \times L\\
    &\leq \frac{\delta}{2} + \frac{nL}{N} < \delta 
\end{align*} % \qedhere
\end{proof}

 {
\begin{remark}\label{rmk.simpleV}
    By the proof of Lemma~\ref{lm.simpleV}, we see that $k > k_0=\sum_{i=1}^{nm} |b_i| > nm$. Typically $b_i$'s are assumed to be larger than $1$,  {which results in increased dimensionality. This is the price for having \emph{both} a full-cyclic permutation state coupling \emph{and} restricted binary complex input weights}.
\end{remark}
}

\begin{corollary}\label{cor.simpleV} 
For any $n\times m$ complex matrix $V$ and $\delta>0$, there exists $k$ matrices $\{F_1, \cdots, F_k\}$ where each $F_i\in \mathbb{M}_{n \times m}\left(\pm 1\right)$ or $\mathbb{M}_{n \times m}\left(\pm i\right)$ and a constant integer $N>0$ such that:
% \commentred{Should be $\mathbb{M}_{n \times m}\left(\pm 1, \pm i \right)$? or are there no mixing? Re: we can achieve no mixing, so each row is either $\pm 1$ or $\pm i$. $M(\pm1, \pm i)$ is weaker than what we got (since a row may have both $\pm 1$ and $\pm i$).} 
\[\norm{V-\frac{1}{N} \sum_{j=1}^k F_j} < \delta\]
Moreover, $k$ can be chosen such that $\gcd(k,n)=1$. 
\end{corollary}
One can easily modify the proof of Lemma~\ref{lm.simpleV} to prove this Corollary. The proof is omitted here. 

We now prove that Complex Simple Cycle Reservoirs are universal.

\begin{theorem}\label{thm.main.cscr} For any reservoir system $R=(W,V,h)$ of dimensions $(n,m,d)$ that satisfies the assumptions of Definition~\ref{def.lrc} and any $\epsilon>0$, there exists a $\mathbb{C}$-SCR $R'=(W',V', h')$ of dimension $(n',m,d)$ that is $\epsilon$-close to $R$. 
Moreover, $\norm{W}=\norm{W'}$ and  $h'$ is $h$ with linearly transformed domain. 
\end{theorem}

\begin{proof}
% [Proof of Theorem~\ref{thm.main.cscr}] 
Consider a reservoir system $R=(W,V,h)$ with dimensions $(n,m,d)$. Without loss of generality, we assume that $W$ is a contractive full-cycle permutation. For any $\epsilon>0$, pick $\delta>0$ such that $\norm{\Bx-\Bx'}<\delta$ implies $|h(\Bx)-h(\Bx')|<\epsilon$. We now construct a $\mathbb{C}$-SCR $R'=(W',V',h')$ that is $\epsilon$-close to $R$. 

 Applying Corollary~\ref{cor.simpleV}, we obtain $n\times m$ matrices  $\left\{F_j\right\}_{j=1}^k$ whose entries are either all $\pm 1$ or $\pm i$, and $N>0$ sufficiently large such that 
\[\norm{V-\frac{1}{N} \sum_{j=1}^k F_j} < \frac{(1-\lambda)\delta}{M},\]
where $M=\sup\{\norm{\Bc_t}\}_{t\in \mathbb{Z}_-}$, and $k$ satisfies $\gcd(k,n)=1$. Let $W=\lambda P$ with $\lambda=\|W\|$ and $P$ being a full-cycle permutation. Applying Lemma~\ref{lm.fullcycleW} we obtain a full-cycle permutation  $nk\times nk$ matrix $P_1$. Consider the reservoir system $R'=(W',V',h')$ defined by the triplet: 
\begin{align*}
W'=\lambda P_1 = \lambda \cdot \begin{bmatrix} 
0   & 0  & 0 & & \hdots      & 0   & P \\
P & 0  & 0 & & \hdots      & 0   & 0 \\
 0   & P &  0     & \hdots   & & 0 & 0 \\
 \vdots   &   & \ddots&    &   &  \vdots & \vdots\\
 0   & \hdots  &  &  &   & P  & 0 
\end{bmatrix}, &\quad V'=\begin{bmatrix} F_1 \\ F_2 \\ \vdots \\ F_k\end{bmatrix}\\
h'\left({\Bx^{(1)}}, \dots, {\Bx^{(k)}}\right) &= h\left(\frac{1}{N}\sum_{j=1}^k {\Bx^{(j)}}\right).
\end{align*}

For any input stream $\{\Bc_t\}_{t\in \mathbb{Z}_-}$, the solution to $R$ and $R'$ are given respectively by:
\[{\By}_t = h\left(\sum_{i\geq 0} W^i V \Bc_{t-i}\right),\quad \By'_t = h'\left(\sum_{i\geq 0} \left(W'\right)^i V' \Bc_{t-i}\right).\]
Note that, by construction $W'$ cycles through $k$ subspaces with dimension $n$ and applies $W$ on each of them. Thus,
\[\left(W'\right)^i V' \Bc_{t-i}= \begin{bmatrix} W^i F_{\left(1+i \mod k\right)} \Bc_{t-i} \\ W^i F_{\left(2+i \mod k\right)} \Bc_{t-i} \\ \vdots \\ W^i F_{\left(k+i \mod k\right)} \Bc_{t-i} \end{bmatrix}.\]
Since for each $i$, $\{F_{\left(1+i \mod k\right)}, F_{\left(2+i \mod k\right)}, \dots, F_{\left(k+i \mod k\right)}\}$ is simply a permutation of $\{F_1, \dots, F_k\}$. We obtain: 
\begin{align*}
\By'_t &= h\left(\frac{1}{N} \sum_{j=1}^k \sum_{i\geq 0}  W^i F_{\left(j+i \mod k\right)} \Bc_{t-i}\right) \\
&= h\left( \sum_{i\geq 0} W^i \frac{1}{N} \left(\sum_{j=1}^k  F_{j}\right) \Bc_{t-i}\right),
\end{align*}

%\begin{align*}
%    {y'}_t &= h\left(\frac{1}{N} \sum_{i\geq 0} \sum_{j=1}^k W^i F_{j+i} c_{t-i}\right) \\
%    &= h\left( \sum_{i\geq 0} W^i \frac{1}{N} \left(\sum_{j=1}^k  F_{j}\right) c_{t-i}\right) \\
%\end{align*}
Now let $\Bx'_t=\sum_{i\geq 0} W^i \frac{1}{N} \left(\sum_{j=1}^k  F_{j}\right) \Bc_{t-i}$ and $ {\Bx}_t=\sum_{i\geq 0} W^i V \Bc_{t-i}$. We have,
\begin{align*}
    \norm{{\Bx}_t-\Bx'_t} &= \norm{\sum_{i\geq 0} W^i \left(V - \frac{1}{N} \sum_{j=1}^k F_j\right) \Bc_{t-i}} \\
    &\leq \sum_{j\geq 0} \lambda^i M \norm{V - \frac{1}{N} \sum_{j=1}^k F_j}\\
    &< \frac{1}{1-\lambda} M \frac{(1-\lambda)\delta}{M} = \delta
\end{align*}
Therefore ${\By}_t$ is $\epsilon$-close to $\By'_t$ by continuity of $h$ and thus $R'$ is $\epsilon$-close to $R$. By construction, $W'$ is a contractive full-cycle permutation and entries of $V'$ are either all $\pm 1$ or $\pm i$. 
\end{proof}

Using a similar argument, we can also prove that an assembly of two SCR (Simple Cycle Reservoir over $\mathbb{R}$) is universal. The argument is to apply the same process for $\operatorname{Re}(V)$ and $\operatorname{Im}(V)$ respectively to  {get a Multi-Cycle Reservoir of order 2.}
%n assembly of two SCR. 

\begin{theorem}\label{thm.main.rscr} For any reservoir system $R=(W,V,h)$ of dimensions $(n,m,d)$ and any $\epsilon>0$, there exists a Twin Simple Cycle Reservoir $R'=(W',V', h')$ of dimension $(n',m,d)$ that is $\epsilon$-close to $R$. 
Moreover, $\norm{W}=\norm{W'}$ and  $h'$ is $h$ with linearly transformed domain. 
\end{theorem}

\begin{proof}
% [Proof of Theorem~\ref{thm.main.rscr}] 
 {Consider the reservoir system $R=(W,V,h)$ where $W$ is a contractive full-cycle permutation.} Write $V=V_r+iV_i$ where $V_r, V_i$ are real and imaginary parts of $V$ respectively. For any $\epsilon>0$, pick $\delta>0$ such that $\norm{\x-\x'}<\delta$ implies $|h(\Bx)-h(\Bx')|<\epsilon$. We now  {construct a Multi-Cycle Reservoir of order 2, $R'=(W',V',h')$,} that is $\epsilon$-close to $R$.

Apply Lemma~\ref{lm.simpleV} on $V_r$ and $V_i$ to obtain constants $N_r, N_i > 0$ and \textit{real-valued} matrices $\{F_1,\dots, F_{k_r}\}$ and $\{G_1, \dots, G_{k_i}\}$ whose entries are all $\pm 1$, such that
\begin{align*}
    \norm{V_r-\frac{1}{N_r} \sum_{j=1}^{k_r} F_j} &< \frac{(1-\lambda)\delta}{2M}, \\
    \norm{V_i-\frac{1}{N_i} \sum_{j=1}^{k_i} G_j} &< \frac{(1-\lambda)\delta}{2M},
\end{align*}
 {where $k_r, k_i > 0$ are chosen such that $\gcd(k_r, n)=\gcd(k_i, n)=1$.} Let $W=\lambda P$ for a full-cycle permutation $P$ and $\lambda=\|W\|$. Apply Lemma~\ref{lm.fullcycleW} \textit{twice} to obtain a $nk_r\times nk_r$ full-cycle permutation $P_r$ and a $nk_i\times nk_i$ full-cycle permutation $P_i$. Consider reservoir system $R'=(W',V',h')$ defined by the triplet:
\begin{align*}
W'=\lambda (P_r \oplus P_i), &\quad 
V'=\begin{bmatrix} F_1 \\ \vdots \\ F_{k_r} \\ G_1 \\ \vdots \\ G_{k_i}\end{bmatrix}, \\
h'(\Bx^{(1)}, \dots, \Bx^{(k_r)}, \Bx'^{(1)}, \dots, \Bx'^{(k_i)}) &= h\left(\frac{1}{N_r} \sum_{j=1}^{k_r} \Bx^{(j)} +\frac{i}{N_i} \sum_{j=1}^{k_i} \Bx'^{(j)} \right)
\end{align*}
 {$R'=(W',V',h')$ is a Multi-Cycle Reservoir of order 2}. Consider the state of the system given by: 
\begin{align*}
    \Bx_t &=\sum_{n\geq 0} W^n V \Bc_{t-n} \\
    &= \sum_{n\geq 0} W^n (V_r+iV_i) \Bc_{t-n} = \sum_{n\geq 0} W^n V_r \Bc_{t-n} + i\cdot \sum_{n\geq 0} W^n V_i \Bc_{t-n}.
\end{align*}
The rest of the proof is similar to that of Theorem~\ref{thm.main.cscr}. The first half of the sum $ \sum_{n\geq 0} W^n V_r \Bc_{t-n}$ can be arbitrarily approximated by 
\[ \sum_{n\geq 0} (\lambda P_r)^n \frac{1}{N_r} \left(\sum_{j=1}^{N_r} F_j\right) \Bc_{t-n},
\]
and by a symmetric argument the second half of the sum $ \sum_{n\geq 0} W^n V_i \Bc_{t-n}$ can be arbitrarily approximated by 
\[
\sum_{n\geq 0} (\lambda P_i)^n \frac{1}{N_i} \left(\sum_{j=1}^{N_i} G_j\right) \Bc_{t-n}.
\]Therefore it follows from a similar argument of the proof of Theorem~\ref{thm.main.cscr} that $\Bx_t$ is close to $\Bx_t'$, and $R'$ is $\epsilon$-close to $R$. 
\end{proof}

\section{Summary and Universality in the Space of Fading Memory Filters}
\label{sec:summary}

 {Having finished our exploration of universality properties of simple reservoir structures employing only scaled full-cycle permutations in the dynamic coupling and binary input-to-state coupling, we now present in Figure \ref{fig:fullpaper} a birds-eye overview of the results and argumentation flow presented so far.
Each arrow represents an approximation step
with the symbol $\prec$ indicating an increase in the approximant state space dimensionality. 
% at the end point of the arrow. 
By Definition~\ref{def.lrc}, the symbols $W, V, h, n$ denote the dynamic coupling, input-to-state coupling, readout, and dimension of the state space, respectively of the corresponding reservoir system. 
}

{
We begin with an arbitrary linear reservoir system $R=(W,V,h)$ in the top-left-hand-corner. We first construct, by Theorem~\ref{thm.dilation}, a linear reservoir system $R_U$ with a unitary dynamic coupling $W_U$ which is $\epsilon$-close to $R$. By Theorem~\ref{thm.to.permutation}, we then transform $R_U$ into an $\epsilon$-close linear reservoir system $R'_U$ with contractive cyclic-permutation dynamic coupling. 
Based on this, given any linear reservoir system $R$, we can construct an $\epsilon$-close linear reservoir system which is a \PermutationReservoir (Theorem~\ref{thm.main.pr}), Complex Simple Cycle Reservoir (Theorem~\ref{thm.main.cscr}), or a Twin Simple Cycle Reservoir (Theorem~\ref{thm.main.rscr}) respectively.
}

\begin{center}
\begin{figure}[ht!]
%  \includestandalone[width=\textwidth]{tikz/full_flow}%     without .tex extension
  % or use \input{mytikz}
  \begin{adjustbox}{width=0.9\textwidth}
  \begin{tikzpicture}[
    pre/.style={=stealth',semithick},
    post/.style={->,shorten >=1pt,>=stealth',thick}
    ]

 % Nodes
 \node[black] (orig) at (-7,3) {$\begin{cases}
     R:=\left(W, V, h\right)\\
     W \in \mathbb{C}_{n\times n} \\
     V \in \mathbb{C}_{m \times n} \\
     h:\mathbb{C}^n \to \mathbb{C}^d\\
     \lambda:= \norm{W}\\
 \end{cases}$};
 
 \node[black] (unit) at (1,3) {$\begin{cases}
 \text{ \textbf{Unitary universal} }\\
 R_U := \left(W_U, V_U, h_U\right)& \\
 W_U := \lambda \cdot U \in \mathbb{C}_{\left(N+1\right)n\times \left(N+1\right)n} \\
 U := \begin{bmatrix}
        W   &   &       &    & D_{W^*} \\
        D_W &   &       &    & -W^* \\
            & I &       &    & \\
            &   & \ddots&    &   \\
            &   &       & I  & 0 
        \end{bmatrix},\quad 
 V_U: =\begin{bmatrix}
    V  \\
    0
    \end{bmatrix} \\
 h_U(x) = h\left(P_n(x)\right) \\
 n_U := \left(N+1\right)n
 \end{cases}$};
 
 \node[black] (perm) at (-5.5,-4) {$\begin{cases}
 \text{ \textbf{Cyclic Permutation universal} }\\
 R'_U := \left(W'_U, V'_U, h'_U\right)  \\
 W'_U = \lambda \cdot \begin{bmatrix} U & 0 \\ 0 & D\end{bmatrix} \cong \lambda \cdot P, \\
 P \text{ -- cyclic permutation, } P \in \mathbb{C}_{n'_U  \times n'_U } \\
 V'_U: =S \begin{bmatrix}
    V_U  \\
    0
    \end{bmatrix},\quad 
 h'_U(\Bx) = h_U\left(P_{n'_U}\left(S^*\Bx\right)\right) \\
 n'_U > n_U,\, S \text{-- unitary transform}
 \end{cases}$};

 \node[black] (permpm1) at (-5.5,-10) {$\begin{cases}
 \text{ \textbf{SMCR universal} }\\
 R_P := \left(W_P, V_P, h_P\right) & \\
 W_P \in \mathbb{C}_{\left(n'_U\right)^2\cdot m \times \left(n'_U\right)^2\cdot m} \\
 W_P \text{ -- contractive permutation.}\\
 W_P := \begin{bmatrix} \lambda\cdot P & & \\
    & \ddots & \\
    & & \lambda\cdot P\end{bmatrix} \\
V_P \in \mathbb{M}_{m \times \left(n'_U\right)^2\cdot m}\left(\left\{-1,1\right\}\right) \\
h_P(\textbf{x}_1, \cdots, \textbf{x}_{n'_U  m}) := h'_U\left( \sum_{i=1}^{n'_U  m} a_i \textbf{x}_i\right) \\
n_p = n'_U \cdot \left(n'_U  m\right)
 \end{cases}$};

 \node[black] (cscr) at (5,-4) {$\begin{cases}
 \text{ \textbf{$\mathbb{C}$-SCR universal} }\\
 R_{\mathbb{C}} := \left(W_{\mathbb{C}}, V_{\mathbb{C}}, h_{\mathbb{C}}\right) & \\
 W_{\mathbb{C}} := \lambda \cdot P_1 \in \mathbb{C}_{n'_U \cdot k \times n'_U \cdot k} \\
 P_1 := \begin{bmatrix} 
0   & 0  & 0 & & \hdots      & 0   & P \\
P & 0  & 0 & & \hdots      & 0   & 0 \\
 0   & P &  0     & \hdots   & & 0 & 0 \\
 \vdots   &   & \ddots&    &   &  \vdots & \vdots\\
 0   & \hdots  &  &  &   & P  & 0 
\end{bmatrix} \cdots \left(\dagger\right) \\
V_{\mathbb{C}} \in \mathbb{M}_{m \times n'_U \cdot k}\left(\left\{-1,1\right\}\right) \text{ OR } V_{\mathbb{C}} \in \mathbb{M}_{m \times n'_U \cdot k}\left(\left\{-i,i\right\}\right) \\
h_{\mathbb{C}}({\Bx_1}, \dots, {\Bx_k}) = h'_U\left(\frac{1}{N^\mathbb{C}}\sum_{j=1}^k {\Bx_j}\right) \\
n_\mathbb{C} = n'_U \cdot k \text{; } k \text{ satisfies } \gcd(k,n'_U ) = 1
 \end{cases}$};

 \node[black] (rscr) at (4,-11) {$\begin{cases}
  \text{ \textbf{Twin SCR universal} }\\
 R_{\mathbb{R}} := \left(W_{\mathbb{R}}, V_{\mathbb{R}}, h_{\mathbb{R}} \right)& \\
 W_{\mathbb{R}} := \lambda (P_r \oplus P_i) = \lambda \cdot \begin{bmatrix} P_r & 0 \\ 0 & P_i \end{bmatrix},\\
 \text{where both } P_r, P_i \text{, has form } \left(\dagger\right).  \\
 W_{\mathbb{R}} \in \mathbb{C}_{n'_U\cdot (k_r + k_i) \times  n'_U\cdot (k_r + k_i)}\\
 V_{\mathbb{R}} \in \mathbb{M}_{m \times  n'_U \cdot (k_r + k_i)}\left(\left\{-1,1\right\}\right)\\
 h_{\mathbb{R}}(\Bx_1, \dots, {\Bx_{k_r}}, {\Bx'_1}, \dots, {\Bx'_{k_i}}) \\
 \quad \quad = h'_U\left(\frac{1}{N^\mathbb{R}_r} \sum_{j=1}^{k_r} {\Bx_j} + \frac{i}{N^\mathbb{R}_i} \sum_{j=1}^{k_i} {\Bx'_j} \right) \\
 n_\mathbb{R} = n'_U\cdot (k_r + k_i) \\
 k_r,k_i \text{ satisfies } \gcd(k_r,n'_U) = \gcd(k_r,n'_U) = 1
 \end{cases}$};
 
 % Arrows
 \path [->,thick] (orig) edge node [above] {Thm \ref{thm.dilation}} (unit);
 \path [->,thick] (orig) edge node [below] {$\prec$} (unit);

 % \draw[post,rounded corners=5pt] (unit)-|(below_orig_perm);
 % \draw[post,rounded corners=5pt] (below_orig_perm)|-(perm);
 
 \draw[post,rounded corners=5pt] (unit) -- node[below] {} ++(0,-4 + 0.5) -| node[below] {} ++(-6.5 ,-0.5) -- (perm);
 \draw[-, ultra thin] (-3,-1 + 0.5) --node [below] {$\prec$}  (0,-1 + 0.5);
 \draw[-, ultra thin] (-3,-1 + 0.5) --node [above] {Thm \ref{thm.to.permutation}}  (0,-1 + 0.5);
 % \draw[post,rounded corners=5pt] (orig)|-(rscr); % final arrow

 % Arrows
 \path [->,thick] (perm) edge node [left] { Thm \ref{thm.main.pr}} (permpm1);
 \path [->,thick] (perm) edge node [right] {$\prec$} (permpm1);

 \path [->,thick] (perm) edge node [above] {Thm \ref{thm.main.cscr}} (cscr);
 \path [->,thick] (perm) edge node [below] {$\prec$} (cscr);
 
 % \path [->] (cscr) edge node [left] {Thm \ref{thm.main.2}} (rscr);
 % \path [->] (cscr) edge node [right] {$\prec$} (rscr);

 \draw[post,rounded corners=5pt] (perm) -- node[below] {} ++(3,-3.5) |- node[above] {} ++(0 ,-2) |-(rscr);
 \draw[-, ultra thin] (-2,-11) --node [below] {$\prec$}  (-0.5 ,-11);
 \draw[-, ultra thin] (-2,-11) --node [above] {Thm \ref{thm.main.rscr}}  (-0.5 ,-11);
 
 % \draw[->] (2.5,3) --node [above] {Prop. \ref{prop.perturb.unitary}} (4,3);
 % \draw[->] (2.5,3) --node [above] {Prop. \ref{prop.perturb.unitary}} (4,3);

\end{tikzpicture}
\end{adjustbox}
  \caption{Detailed flow of the main results of this paper}
  \label{fig:fullpaper}
\end{figure}
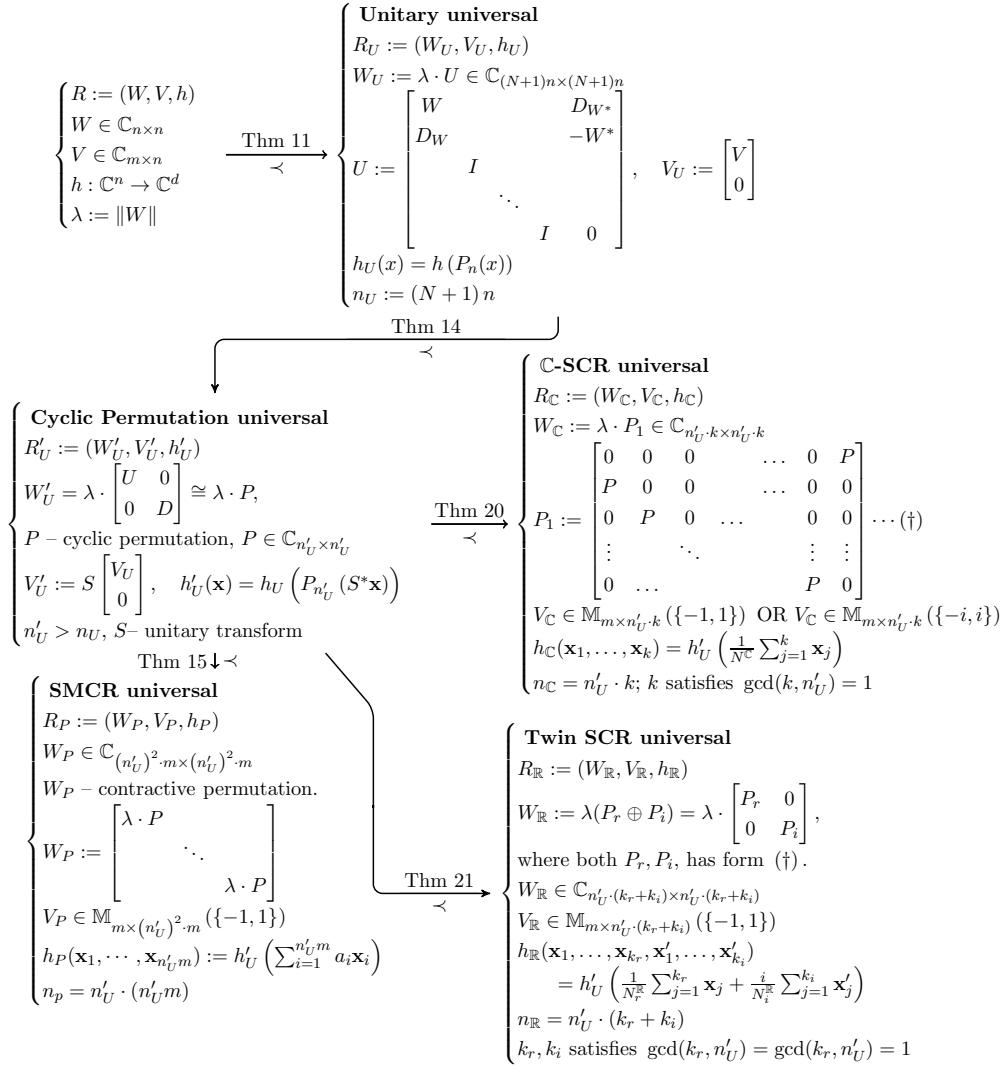
\end{center}

 {
Combining this result with Corollary 11 of \cite{Grigoryeva2018} implies that all three types of linear reservoir systems described in Definition~\ref{def.rc} are universal in the category of time-invariant fading memory filters.
We illustrate this with $\mathbb{C}$-SCR as the cases of \PermutationReservoir \ and Twin SCR follow the same argumentation. 

\cite{Grigoryeva2018}(Corollary 11) establishes that linear reservoir systems with polynomial readouts are universal, in the sense that any time-invariant fading memory filter can be approximated by to arbitrary precision by a linear reservoir system. 
%The main results of this paper demonstrates the universality of several much smaller classes of linear reservoir systems.  This, combined with the result in \cite[Corollary 11]{Grigoryeva2018}, implies the universality of these classes of reservoir systems among time-invariant fading memory filters. 
In other words,  given any time-invariant fading-memory filter $F$ and $\epsilon >0$, \emph{there exists} a linear reservoir system $R$ with polynomial readout $h$ and the corresponding linear reservoir function $H_R$ such that $H_R$ is $\epsilon$-close to $F$ in the space of real-valued continuous functions over the space of uniformly bounded input. 
By Theorem~\ref{thm.main.cscr}, given a linear reservoir system $R$ with a polynomial readout $h$, we can \emph{construct} a $\mathbb{C}$-SCR  $R'$ that is $\epsilon$-close to $R$ in the space of linear reservoir systems. Moreover, the readout of $R'$, denoted by $h'$, is $h$ with linearly transformed domain,
meaning that $h'$ is a polynomial of the same degree as $h$. 
It is worth noting that our results are not restricted to polynomial readouts, as long as they are continuous.
}

%We therefore arrive at the following theorem:  {Not sure if this is the correct level of explicit-ness}

\begin{theorem}
\label{thm.univ}
    { Any time-invariant fading memory filter over uniformly bounded inputs can be 
   % ( {uniformly - WHAT DOES IT MEAN HERE?}) \comment{The assumptions of both papers is that the inputs is a uniformly bounded stream.}  {[THAT I UNDERSTAND AND IT IS THE FIRST ``UNIFORMLY IN THE SENTENCE, BUT I WAS WONDERING WHAT THE SECOND ``UNIFORMLY" MEANS IN ... can be uniformly approximated ...]}
   approximated to arbitrary precision by a \PermutationReservoir, a $\mathbb{C}$-SCR, or a Twin SCR, each endowed with a polynomial readout.}  
\end{theorem}

%  {YES, IT IS GOOD TO BE PRECISE. GIVEN A TIME-INVARIANT FADING MEMORY FILTER, \cite{Grigoryeva2018} SHOW THAT THERE IS A LINEAR RESERVOIR SYSTEM $R$ WITH POLYNOMIAL READOUT $h$ THAT APPROXIMATES IT ARBITRARILY WELL. WE THEN SHOW THAT THERE IS A RESERVOIR $R'$ WITH CYCLIC PERMUTATION STATE COUPLING, RESTRICTED BINARY COMPLEX INPUT WEIGHTS (of either $\pm 1$ or $\pm i$) AND POLYNOMIAL READOUT $h'$ (IN FACT $h'$ IS $h$ WITH LINEARLY ADAPTED DOMAIN) THAT APPROXIMATES $R$ ARBITRARILY WELL. HOWEVER, OUR THEORY IS NOT RESTRICTED TO POLYNOMIAL READOUTS, AS LONG AS THEY ARE CONTINUOUS.}

\section{Conclusion}

 {
We have shown that even severely restricted linear reservoir architectures with continuous readouts, employing only scaled full-cycle permutation dynamic couplings and binary input-to-state couplings are capable of universal approximation of any unrestricted linear reservoir system (with continuous readout) and hence any time-invariant fading memory filter over uniformly bounded input streams. 

These results support empirical studies reporting the competitive performance of simple cyclic reservoir structures (e.g.
\cite{rodan2010minimum,WANG2019184,1363951794068753792}), as well as theoretical investigations of the representational power of such architectures in terms of memory and state space organisation
(\cite{Tino_JMLR_2020,rodan2010minimum}).  

Universality guarantees of simple reservoir architectures that lend themselves naturally to physical implementations (\cite{bienstman2017,Appeltant2011InformationPU,NTT_cyclic_RC}) represent an important step in transferring reservoir computation ideas to real-world and industrial applications.  
}

\clearpage
\clearpage

\bibliography{references}

\end{document}